\documentclass{article}
\usepackage{tikz} 
 \usepackage[preprint]{neurips_2026}

\usepackage{graphicx}
\graphicspath{ {./images/} }
\usepackage[ruled]{algorithm2e}
\usepackage{subcaption} 

\usepackage{adjustbox}
\usepackage{multicol}
\usepackage{multirow}
\usepackage{tabularx}

\usepackage{xcolor}
\usepackage{color}
\usepackage{colortbl} 
\usepackage{inconsolata}
\usepackage{xcolor}
\definecolor{asparagus}{rgb}{0.53, 0.66, 0.42}

\usepackage{physics}
\usepackage{mathtools}
\usepackage{amsmath}
\usepackage{amssymb}

\usepackage[utf8]{inputenc} 
\usepackage[T1]{fontenc}    
\usepackage{hyperref}       
\usepackage{url}            
\usepackage{booktabs}       
\usepackage{amsfonts}       
\usepackage{nicefrac}       
\usepackage{microtype}      
\usepackage{xcolor}         

\usepackage{amsthm}
\newtheorem{theorem}{Theorem}
\newtheorem{definition}{Definition}
\newtheorem{corollary}{Corollary}
\newtheorem{lemma}{Lemma}

\newtheorem{assumption}{Assumption}
\makeatletter
\@addtoreset{theorem}{section}
\@addtoreset{definition}{section}
\@addtoreset{corollary}{section}
\@addtoreset{lemma}{section}
\@addtoreset{claim}{section}
\@addtoreset{property}{section}
\@addtoreset{remark}{section}
\makeatother

\usepackage{caption}
\usepackage{subcaption}

\usepackage{multicol}
\usepackage{wrapfig}
\usepackage{blindtext}

\usepackage{mdframed}

\usepackage{enumitem}

\title{Position: Avoid Overstretching LLMs for every Enterprise Task}

\author{
  Kuldeep Singh*\\
  Eka Labs AI,\\
  Germany\\
  \And
  Anson Bastos*\\
  Microsoft,\\
  India\\
  \And
  Isaiah Onando Mulang'\\
  SAP,\\
  Germany
}

\begin{document}

\maketitle

\begin{abstract}
Enterprise workloads are dominated by deterministic, structured, and knowledge-dependent tasks operating under strict cost, latency, and reliability constraints. While these are often addressed through large language model (LLM) deployment or distillation into smaller models, we argue this is inefficient, unreliable, and misaligned with enterprise task structures.
Instead, AI systems should treat language models as interfaces rather than monolithic engines, externalizing knowledge and computation into dedicated components for greater reliability, scalability, and transparency. Our theoretical evidences show that finite-capacity models cannot fully capture the breadth of knowledge required for enterprise tasks, creating inherent limits to efficiency and interpretability. Building on this, we take the position that 
language models should primarily be used for structured extraction in deterministic enterprise workflows,
while computation and storage are delegated to knowledge bases and symbolic procedures. We formally demonstrate that such modular architectures are more reliable and maintainable than monolithic frameworks, offering a sustainable foundation for enterprise tasks.

\end{abstract}

\section{Introduction}

Enterprise AI systems increasingly sit at the core of critical operational workflows, including deterministic patterns, policy compliance, incident triage, knowledge‑base question answering, and structured diagnostics. These workloads are characterized by repetitive structure, strong domain regularities with explicit constraints on latency, cost, privacy, and reliability. Despite this, recent trends suggest solving these problems by directly deploying large language models (LLMs) or by distilling their behavior into smaller models. While effective in some settings, this strategy raises fundamental questions of efficiency and methodological alignment with the underlying task structure.

At a practical level, LLM inference is expensive and difficult to scale under enterprise constraints, especially for continuous, high‑volume, and latency‑sensitive workloads. Methodologically, parametric language models are increasingly treated as monolithic repositories of both world knowledge and reasoning procedures, implicitly assuming that sufficient scale can collapse storage, retrieval, and computation into a single learned function. Even in agentic settings that ostensibly introduce reasoning, planning, grounding, and tool orchestration, these monolithic repertoires remain heavily deployed in multi-turn queries in which a single LLM serves as an oracle for several specialized tasks and an orchestrator of tools and memory. 
This design pattern resembles earlier monolithic software systems, where a single component attempts to internalize all functionality, whereas modern distributed systems evolved toward microservice architectures that separate concerns across specialized, independently governed components. 
Although current agentic methods incorporate retrieval and tool usage, these capabilities are typically orchestrated by an \emph{LLM acting as a unified policy engine}, leaving knowledge access and computation \emph{entangled within a shared parameterization}.
In this paper, we argue that this assumption is not merely inefficient but fundamentally flawed for a broad class of enterprise constrained deterministic tasks.
We propose a \emph{paradigm inversion}: external systems and explicit workflows should govern model invocation, relegating specialized SLMs to bounded interface roles rather than central orchestrators in structured enterprise tasks.

\paragraph{Position:}
We argue that directly deploying frontier LLMs, or distilling them into small language models (SLMs), is fundamentally misaligned with enterprise requirements for reliability and governance. 
\textbf{Instead, we advocate an \emph{architectural separation} in which knowledge and recurrent computation are externalized into inspectable and controllable components (e.g., databases, knowledge bases, and symbolic procedures), while using SLMs as lightweight extraction and interface layers.}
We further postulate that, for knowledge-intensive and recurrent algorithmic enterprise tasks, such hybrid architectures provide a more scalable and reliable foundation than purely parametric approaches that attempt to encode both knowledge and reasoning within large model parameters.
Our contribution is not the use of tools or retrieval per se, but the inversion of control: shifting runtime decision authority from a monolithic LLM to explicitly governed system components, while automating semi-deterministic enterprise workflows.

\textbf{Background:} In this paper, we primarily target enterprise tasks that need to run autonomously with no or minimal human in the loop. A central paradox characterizes the current phase of enterprise AI adoption: deployment is widespread, yet measurable enterprise-scale value creation remains limited. 
Across global surveys conducted between 2025 and 2026, the majority of organizations are deploying AI in at least one business function and continuing to increase investment in AI systems \citep{McKinsey2025StateOfAI, Deloitte2026StateOfAI, OpenAI2025StateEnterpriseAI}. 
However, only a minority of organizations achieve enterprise-wide deployment or measurable financial impact, with many initiatives confined to localized pilots or experimentation \citep{BizzDesign2026EnterpriseAIAdoption, PepperFoster2025AIROI}. 
The gap between adoption and value creation has become a defining theme of enterprise AI research.
Consequently, even when operational efficiency gains are observed locally, these outcomes often fail to propagate to enterprise-scale financial or operational impact \citep{McKinsey2025StateOfAI, WhartonVistage2025EnterpriseAIROI}. 
Industry surveys consistently identify the transition from pilot deployment to reliable production systems \citep{ISG2025EnterpriseAIAdoption, walkme_digital_adoption_2025}.
A closer examination reveals that barriers to adoption are largely structural rather than purely algorithmic. \textit{Technically}, challenges include integrating AI with legacy infrastructure, managing sensitive proprietary data, and addressing persistent concerns regarding reliability and hallucinations \cite{opentext_capgemini_wqr_2025, Deloitte2025AITrendsPulse}. \textit{Organizationally}, firms encounter skill shortages, fragmented ownership across business units, and divergent perceptions between management and the workforce regarding AI readiness \cite{ HBR2025OrgBarriersAI}.

\textbf{The Core Mismatch:} These structural constraints highlight a deeper misalignment between current enterprise AI practices and the nature of enterprise workloads.
Unlike open-ended consumer or research settings, enterprise tasks are typically governed by explicit workflows, domain-specific knowledge, compliance obligations, and operational accountability.
Across sectors, from finance and healthcare to manufacturing, automotive, and government, AI is increasingly deployed for tasks such as fraud detection, clinical documentation, compliance analysis, process optimization, and decision support \cite{Menlo2025StateAIHealthcare, McKinsey2025StateOfAI}. 
These domains share a defining characteristic: the underlying knowledge sources, constraints, and decision rules are explicit, structured, and continuously evolving.

In this paper, we argue that attempting to encode such knowledge and reasoning entirely within model parameters, whether through large-scale training or through distillation of LLM outputs, conflates distinct computational roles and leads to fragile enterprise deployments.
Instead, we propose that enterprise AI systems should externalize knowledge storage and repeatable algorithmic reasoning to dedicated and inspectable modules, while using language models primarily as structured extraction interfaces between unstructured inputs and the formal systems that govern enterprise workflows. 
We support this position through a set of theoretical results that connect language model capacity to information-theoretic and algorithmic limits. We further show that hybrid architectures combining SLM extraction with external knowledge and symbolic reasoning can outperform monolithic LLM or distilled SLM systems on tasks dependent on external knowledge or structured computation.
In summary, this paper takes a principled position on how language models should be used in enterprise settings. Specifically we make the following contributions while maintaining our position:
\begin{enumerate}[wide, labelindent=0pt, labelwidth=!, leftmargin=*]
    \item \textbf{A Unifying Viewpoint:} We argue that enterprise AI should not be architected around an LLM as the primary policy and reasoning engine augmented by retrieval and tools. Instead, knowledge, computation, and verification must be elevated to first-class system components, with language models serving only as bounded interface layers between unstructured inputs and governed enterprise infrastructure for structured tasks.
    \item  \textbf{Theoretical Foundation for Reliability}: We back our position with a rigorous theoretical basis, demonstrating that enterprise AI reliability improves significantly by combining language models with formal methods. Crucially, we formally prove that current industry-standard distillation techniques are fundamentally information-losing, creating a "projection bottleneck" that prevents SLMs from inheriting complex algorithmic reasoning . By externalizing storage and computation, we ensure models operate within their inherent capacity for governable enterprise process.
    \item \textbf{Architectural Separation:} We advocate for using task-specific, small language models as lightweight extraction interfaces paired with external knowledge layers, enabling scalable, interpretable, and cost-effective business processes that bridge the "pilot-to-production" gap. 
    \item \textbf{The Path Forward:} We propose a hybrid frontier integrating the expressive power of probabilistic models with the provable reliability of deterministic systems, creating a sustainable foundation for industrial intelligence and governed agentic workflows.
\end{enumerate}

\begin{figure}
  \centering
  \includegraphics[width=1\columnwidth]{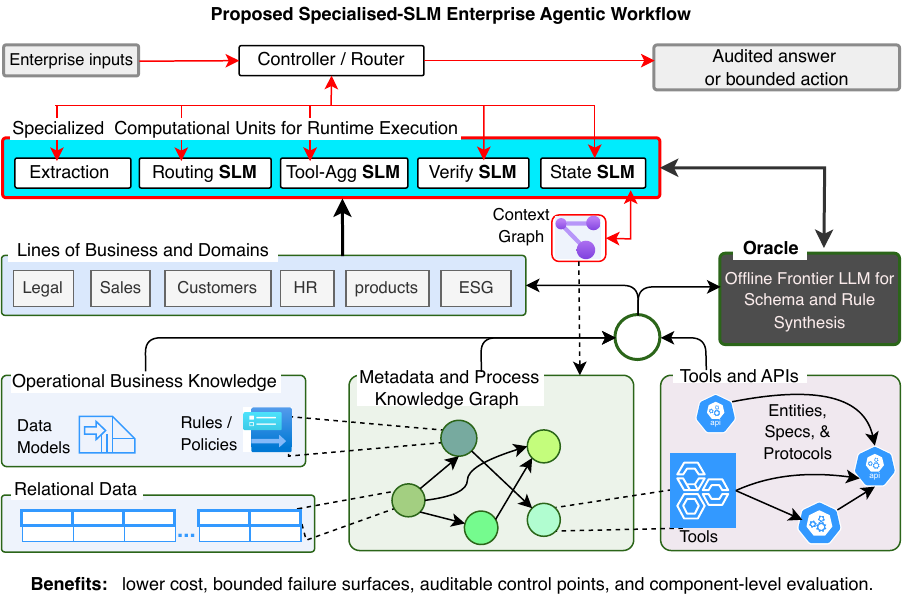}
  \caption{The Proposed Specialised-SLM Enterprise Agentic Workflow (\textcolor{red}{Red arrows indicate live runtime communication}). Unlike monolithic designs using heavily-parameterized LLMs, this architecture assigns specialised SLMs to narrow roles like extraction and routing to minimize error propagation. A controller links these interface units to externalised knowledge graphs and symbolic tools, which carry the substantive burden of retrieval and reasoning. This separation allows probabilistic models to act as interfaces while deterministic components enforce governance. High-capacity frontier LLMs remain offline and act as oracles that synthesise schemas and rules, ensuring the online control loop is cost-effective, auditable, and grounded in explicit infrastructure.}
  \label{fig:agentic_slm_overview}
\end{figure}

\section{Related Work}

A growing body of research examines the gap between enterprise adoption of AI technologies and the realization of enterprise-scale value.
Large-scale industry surveys consistently report that although organizations widely experiment with AI, most deployments remain limited to pilot projects and siloed use cases rather than integrated enterprise systems.
Despite widespread use of AI in single business functions, McKinsey’s global survey finds that nearly two-thirds of organizations have yet to scale these systems, with only a minority achieving measurable enterprise-wide impact \cite{McKinsey2025StateOfAI}. 
Similarly, a large-scale software and digital operations survey shows that despite broad experimentation with generative AI, only a small fraction of organizations have reached enterprise-wide deployment, with technical barriers such as legacy-system integration and data governance cited as major obstacles \cite{opentext_capgemini_wqr_2025}. 
Operational studies also highlight a persistent “pilot-to-production” gap, where organizations deploy AI pilots but struggle to operationalize them at scale because of organizational constraints, governance gaps, and insufficient workforce readiness \cite{walkme_digital_adoption_2025,ey_responsible_ai_2025}.

Adjacent academic literature analyzes the organizational and technical barriers underlying these patterns. 
Empirical and survey-based studies of AI adoption in enterprises identify barriers including inadequate organizational strategy, weak governance frameworks, lack of AI expertise, and challenges integrating AI systems with existing operational processes \cite{hana_ai_adoption_2025,romeo_ai_org_review_2025}. 
In parallel, research on enterprise deployments of large language models highlights difficulties with reliability, hallucination, and integration with existing knowledge infrastructures, motivating architectures that combine LLM capabilities with external retrieval and knowledge systems \cite{rag_enterprise_slr_2025,rag_structured_naacl_2024}. 
Collectively, these studies suggest that the primary obstacles to realizing enterprise value from AI are not merely model capability limitations but systemic challenges related to organizational integration, governance, and workflow compatibility.
We argue that hybrid architectures combining external knowledge systems with lightweight language-model interfaces are most appropriate in these settings.

\section{Theoretical Formalization of the Enterprise AI Design Choice}
\label{sec:basic_formalization}

We briefly formalize the intuition underlying our architectural position: language models with bounded parameters cannot reliably represent all of the knowledge and computational structure required for enterprise tasks.
Instead, enterprise systems benefit from separating language understanding from knowledge storage and algorithmic computation (details and proofs are in Section~\ref{sec:appendix}).

\paragraph{Knowledge and task structure.}
Let $X$ denote the input to the system (e.g., a user query, enterprise document, or event), and $Y$ the desired output.
In many enterprise settings, the correct output depends not only on $X$ but also on task-relevant information stored in external knowledge sources.
We denote this information by $Z$, obtained from a knowledge base $\mathcal{K}$, and model the task by
\begin{equation}
Y = h(X, Z),
\label{eq:enterprise_task}
\end{equation}
where $h(\cdot)$ represents the underlying decision or reasoning procedure.
If the required information $Z$ is not fully determined by $X$, then the conditional mutual information
\begin{equation}
I(Y; Z \mid X) > 0,
\label{eq:cmi_condition}
\end{equation}
which implies that $X$ alone does not suffice to determine the correct output.
Consequently, a system that relies only on parametric inference from $X$ cannot achieve zero error on such tasks.

\paragraph{Parametric inference versus externalized knowledge.}
Consider a language model $f_\theta$ with parameters $\theta$ trained to produce outputs directly from $X$:
\begin{equation}
\widehat{Y}_{\text{LLM}} = f_\theta(X).
\end{equation}
If $Z$ contains additional task-relevant information beyond $X$, any model that depends only on $X$ and $\theta$ faces an irreducible error floor determined by the conditional uncertainty of $Y$ given $X$.
In contrast, a modular system that retrieves task-relevant information and applies an explicit reasoning procedure computes
\begin{equation}
\widehat{Y}_{\text{Hybrid}} = R\big(\phi(X), \mathcal{K}\big),
\label{eq:hybrid_system}
\end{equation}
where $\phi$ is a lightweight extraction function (implemented by a small language model) that maps $X$ to a structured representation, and $R$ is a deterministic algorithmic procedure over the extracted structure and external knowledge $\mathcal{K}$.

\paragraph{Computational Capacity. }
Let \(I_{\mathrm{task}}\) denote the information complexity of a task and let \(I_{\mathrm{model}}\) denote the effective information capacity of a model.
A necessary feasibility condition for fitting the model on the task is
\begin{equation}
I_{\mathrm{model}} \ge I_{\mathrm{task}}.
\label{eq:capacity_condition}
\end{equation}
When \(I_{\mathrm{task}} > I_{\mathrm{model}}\), the model cannot fully represent the task, regardless of optimisation quality.
This is particularly relevant for enterprise workflows requiring explicit rules, multi-step procedures, evolving policies, or out-of-distribution adaptation.

\textbf{Formal takeaway.}
If enterprise tasks depend on external knowledge beyond the input (\(I(Y;Z\mid X)>0\)) or require computation whose information complexity exceeds model capacity (\(I_{\mathrm{task}} > I_{\mathrm{model}}\)), then purely parametric systems are structurally insufficient.
In these regimes, modular systems that externalize knowledge and computation are not merely engineering preferences, but principled architectural requirements.
In the following subsections we delve into the details of the theoretical constructions for interested readers. Skip to \S \ref{position_statement} where we state the positions we take in the paper.


\subsection{A Theoretical Separation: Symbolic Reasoning with SLM Extraction and External Knowledge} \label{sec:theoraticalseperation}

In this section, we establish a formal separation showing that a hybrid system,
consisting of (i) a small language model (SLM) used only for structured extraction,
(ii) an external knowledge source, and (iii) a symbolic reasoner, can be \emph{strictly}
more capable than an SLM distilled from an LLM \emph{without rationales}. Our analysis
is information-theoretic and computational in nature: we formalize (a) an irreducible
error floor induced by missing external information, and (b) an expressivity gap for
algorithmic tasks that are efficiently computable by symbolic procedures but not by
a given bounded-capacity SLM.

\subsubsection{Problem Setup}

Let $(X,Y)\sim\mathcal{D}$ be an input--output distribution for a task of interest.
We assume that the correct output depends on task-relevant facts that are not fully
contained in the input $X$ but are obtainable from an external knowledge source.

\paragraph{External knowledge and task-relevant state.}
Let $\mathcal{K}$ denote an external knowledge source (e.g., a database, knowledge graph,
document store, or curated corpus). Let $Z$ be a random variable denoting the minimal
task-relevant information that must be consulted from $\mathcal{K}$ to answer $X$.
We assume that the task is governed by a deterministic rule
\begin{equation}
Y \;=\; h(X,Z),
\label{eq:Y-hXZ}
\end{equation}
for some (unknown) function $h$.

\paragraph{Symbolic interface.}
We posit an ideal symbolic extraction mapping $\phi^\star$ that maps an input $X$ to a
structured representation (e.g., entities, relations, slots, query form, or logical facts)
used by a symbolic reasoner:
\begin{equation}
\phi^\star:\mathcal{X}\to \mathcal{E}.
\end{equation}
A symbolic reasoner $R$ is then a deterministic algorithm that, given a structured
representation and access to $\mathcal{K}$, produces an output:
\begin{equation}
Y \;=\; R\big(\phi^\star(X),\mathcal{K}\big).
\label{eq:Y-RphiK}
\end{equation}
We interpret~\eqref{eq:Y-RphiK} as a \emph{specification} of the task: if the correct
structure is recovered, the reasoner yields the correct answer.

\subsubsection{Systems Under Comparison}

We compare the following two systems.

\paragraph{(A) Rationale-free distilled SLM.}
Let $S_W$ be a distilled SLM with parameters $W$ trained by response-only distillation,
i.e., the student is trained to match a teacher's outputs $y_T$ (or logits) without
observing any intermediate rationales. At inference, $S_W$ produces
\begin{equation}
\widehat{Y}_S \sim p_W(\cdot\mid X),
\end{equation}
and does \emph{not} consult $\mathcal{K}$. Thus, $\widehat{Y}_S$ is a function of $X$
and $W$ only.

\paragraph{(B) Hybrid symbolic system with SLM extraction.}
Let $\widehat{\phi}$ be an SLM-based extractor that approximates $\phi^\star$. The
hybrid system computes:
\begin{equation}
\widehat{Y}_H \;=\; R\big(\widehat{\phi}(X),\mathcal{K}\big).
\end{equation}
We assume $R$ is \emph{sound} relative to the specification: whenever
$\widehat{\phi}(X)=\phi^\star(X)$ and the knowledge source provides the relevant facts,
then $\widehat{Y}_H = Y$.

\subsubsection{Strict Advantage from Access to External Information}

We begin with an information-theoretic separation: if the label depends on external
knowledge beyond the input, then any model that only observes $X$ suffers an irreducible
Bayes error, whereas a KB-grounded symbolic pipeline can reduce error to the extractor's
failure probability.

\begin{assumption}[Nontrivial external information]\label{ass:nontrivial}
We assume the task depends on knowledge beyond the input in the sense that
\begin{equation}
I(Y;Z\mid X) \;>\; 0.
\label{eq:MI-nontrivial}
\end{equation}
\end{assumption}

Assumption~\ref{ass:nontrivial} states that $Z$ contains information about $Y$ not already
contained in $X$. This holds whenever there exist inputs $x$ for which different KB states
(or retrieved facts) yield different correct outputs.

\begin{lemma}[Irreducible Bayes risk without external knowledge]\label{lem:bayes-floor}
Under Assumption~\ref{ass:nontrivial}, any predictor that only observes $X$ incurs a
nonzero Bayes risk:
\begin{equation}
\mathcal{R}^\star_{X}
\;:=\;
\inf_{\widehat{y}:\mathcal{X}\to\mathcal{Y}}
\Pr\!\big[\widehat{y}(X)\neq Y\big]
\;>\; 0.
\label{eq:bayes-floor}
\end{equation}
\end{lemma}

\begin{proof}
Since $I(Y;Z\mid X)>0$, $Y$ is not a deterministic function of $X$ alone. Equivalently,
there exists a set of inputs with nonzero probability mass for which $p(Y\mid X=x)$ is
non-degenerate. For such $x$, any deterministic predictor $\widehat{y}(x)$ makes an error
with probability at least $1-\max_{y}p(Y=y\mid X=x)$, which is strictly positive. Taking
expectation over $X$ yields $\mathcal{R}^\star_{X}>0$.
\end{proof}

\begin{lemma}[Hybrid error is bounded by extraction failure]\label{lem:hybrid-bound}
Assume the reasoner $R$ is sound given correct extraction and that the knowledge source
returns the task-relevant facts. Define the extraction success event
\[
E := \{\widehat{\phi}(X)=\phi^\star(X)\}.
\]
Then the hybrid system's risk is bounded as
\begin{equation}
\Pr[\widehat{Y}_H\neq Y] \;\le\; \Pr[\neg E].
\label{eq:hybrid-risk-bound}
\end{equation}
\end{lemma}

\begin{proof}
On event $E$, we have $\widehat{\phi}(X)=\phi^\star(X)$, and by soundness of $R$ and
availability of the relevant KB facts, we obtain
\[
\widehat{Y}_H = R(\widehat{\phi}(X),\mathcal{K})
= R(\phi^\star(X),\mathcal{K}) = Y.
\]
Therefore, $\{\widehat{Y}_H\neq Y\}\subseteq \neg E$, which implies~\eqref{eq:hybrid-risk-bound}.
\end{proof}

We now combine these lemmas into a strict dominance result.

\begin{theorem}[Strict advantage via external knowledge]\label{thm:strict-adv-mi}
Under Assumption~\ref{ass:nontrivial}, any rationale-free distilled SLM $S_W$ that does not
consult $\mathcal{K}$ is lower-bounded by $\mathcal{R}^\star_{X}>0$. Moreover, if the
hybrid system satisfies $\Pr[\neg E] < \mathcal{R}^\star_{X}$, then it is strictly better:
\begin{equation}
\Pr[\widehat{Y}_H\neq Y] \;<\; \inf_{W}\Pr[\widehat{Y}_S\neq Y].
\label{eq:strict-adv-mi}
\end{equation}
\end{theorem}

\begin{proof}
Any rationale-free distilled SLM without KB access outputs $\widehat{Y}_S$ as a function of
$X$ (and fixed weights $W$). Thus, its risk is lower-bounded by the Bayes risk given $X$,
namely $\mathcal{R}^\star_X$ from Lemma~\ref{lem:bayes-floor}. Meanwhile, Lemma~\ref{lem:hybrid-bound}
gives $\Pr[\widehat{Y}_H\neq Y]\le \Pr[\neg E]$. If $\Pr[\neg E]<\mathcal{R}^\star_X$, we obtain
the strict inequality~\eqref{eq:strict-adv-mi}.
\end{proof}

\paragraph{Interpretation.}
Theorem~\ref{thm:strict-adv-mi} formalizes a common enterprise regime: when correct answers
depend on up-to-date or proprietary knowledge (captured by $Z$), any model lacking access to
that knowledge faces an irreducible error floor. A KB-grounded symbolic pipeline can surpass
this floor if extraction is sufficiently reliable.

\subsubsection{Strict Advantage from Computational (Expressivity) Separation}

We now present a second separation based on computation. Even if external knowledge were
available, certain tasks require algorithmic procedures (e.g., exact counting, structured
traversals, constrained satisfaction) that may not be representable or systematically
generalizable by a bounded-capacity SLM.

\paragraph{Function classes.}
Let $\mathcal{C}_{\mathrm{SLM}}(B)$ denote the class of functions representable by SLMs up to
capacity $B$ (e.g., bounded parameters/precision/architecture). Let $\mathcal{C}_{R}$ denote
the class of functions computable by the symbolic reasoner $R$ with access to $\mathcal{K}$.

\begin{assumption}[Expressivity separation]\label{ass:sep}
There exists a task function $f^\star$ such that
\begin{equation}
f^\star \in \mathcal{C}_{R}
\qquad\text{but}\qquad
f^\star \notin \mathcal{C}_{\mathrm{SLM}}(B).
\label{eq:sep}
\end{equation}
\end{assumption}

Assumption~\ref{ass:sep} captures regimes where the symbolic procedure can implement the
desired computation exactly, while a bounded SLM cannot represent it (or cannot generalize
it beyond the training regime).

\begin{theorem}[Strict advantage via computational separation]\label{thm:strict-adv-comp}
Under Assumption~\ref{ass:sep}, there exists a distribution $\mathcal{D}$ and a constant
$\epsilon>0$ such that any rationale-free distilled SLM $S_W$ of capacity $B$ incurs error
at least $\epsilon$:
\begin{equation}
\inf_{W:\, S_W \in \mathcal{C}_{\mathrm{SLM}}(B)} \Pr[\widehat{Y}_S \neq f^\star(X,\mathcal{K})]
\;\ge\; \epsilon.
\label{eq:slm-lb}
\end{equation}
If, additionally, the hybrid system computes $\widehat{Y}_H = f^\star(X,\mathcal{K})$ whenever
extraction succeeds, then
\begin{equation}
\Pr[\widehat{Y}_H \neq f^\star(X,\mathcal{K})] \;\le\; \Pr[\neg E].
\end{equation}
Consequently, if $\Pr[\neg E] < \epsilon$, the hybrid system is strictly better than any
rationale-free distilled SLM of capacity $B$.
\end{theorem}

\begin{proof}[Proof sketch]
The separation $f^\star\notin \mathcal{C}_{\mathrm{SLM}}(B)$ implies that no SLM in this
class can realize $f^\star$ on all inputs; equivalently, there exists a distribution (often
over challenging structures or lengths) on which every function in $\mathcal{C}_{\mathrm{SLM}}(B)$
errs with probability at least $\epsilon>0$. This yields~\eqref{eq:slm-lb}. The hybrid system
delegates the target computation to $R$, which computes $f^\star$ exactly given correct extraction
and KB access; thus its error is bounded by extraction failure $\Pr[\neg E]$ as in
Lemma~\ref{lem:hybrid-bound}. If $\Pr[\neg E] < \epsilon$, strict dominance follows.
\end{proof}

\paragraph{Interpretation.}
Theorem~\ref{thm:strict-adv-comp} formalizes that the hybrid pipeline can outperform a rationale-free
distilled SLM not only by providing missing knowledge, but also by providing \emph{algorithmic structure}
through an explicit symbolic reasoner. In particular, the student SLM is not required to internalize
the full reasoning procedure; it only needs to provide a structured interface (entities, relations, or
query forms) to the symbolic module.

\subsubsection{Discussion: Why Rationale-Free Distillation Is Especially Limited}

We emphasize that response-only distillation trains the student to mimic teacher outputs without
transferring intermediate reasoning structure. When tasks require decomposable subroutines or explicit
verification, the student is incentivized to approximate the teacher's mapping by shortcuts that may
not be stable under distribution shift or compositional perturbations. In contrast, our hybrid pipeline
externalizes reasoning into $R$ and externalizes knowledge into $\mathcal{K}$, thereby reducing the
burden on the SLM to encode both facts and algorithms in its weights.

\subsection{Why monolithic LLMs (even with with tools and RAGs) are wasteful vis a vis specialized SLMs: A Theory of Modular Decomposition, Capacity Allocation, and Efficiency}
\label{sec:modular_theory}

We formalize the intuition that enterprise AI systems benefit from separating language understanding, knowledge access, and computation, and show that such decomposition yields both information-theoretic and efficiency gains over monolithic LLM-based systems.

\subsection{Entropic decomposition of enterprise tasks}

Let \(X\) denote the input, \(Y\) the desired output, and \(Z\) the task-relevant information obtained from external knowledge sources. As established earlier, many enterprise tasks satisfy:
\begin{equation}
Y = h(X,Z), \quad \text{with} \quad I(Y;Z \mid X) > 0,
\label{eq:task_external}
\end{equation}
indicating that the input alone does not determine the correct output.

We introduce an entropic decomposition of the conditional uncertainty:
\begin{equation}
H(Y \mid X) = H_{\textsc{know}} + H_{\textsc{comp}} + H_{\textsc{lang}} + \Delta,
\label{eq:entropy_decomposition}
\end{equation}
where:
\begin{itemize}
    \item \(H_{\textsc{know}}\) captures uncertainty due to missing or evolving knowledge,
    \item \(H_{\textsc{comp}}\) captures uncertainty due to algorithmic or procedural computation,
    \item \(H_{\textsc{lang}}\) captures uncertainty in mapping unstructured language to structured representations,
    \item \(\Delta\) captures residual interactions between these components.
\end{itemize}

In monolithic LLM systems, all components must be resolved implicitly within a single parametric model. In contrast, retrieval and tools construct evidence \(E\), reducing:
\begin{equation}
H(Y \mid X,E) \ll H(Y \mid X),
\end{equation}
primarily eliminating \(H_{\textsc{know}}\) and a significant portion of \(H_{\textsc{comp}}\).

\paragraph{Interface-dominated regime.}
In the limit, we obtain:
\begin{equation}
H(Y \mid X,E) \approx H_{\textsc{lang}},
\end{equation}
implying that the remaining task is principally an interface problem.

\subsection{Separation of concerns and complexity reduction}

A monolithic system learns:
\begin{equation}
\hat{Y}_{\textsc{llm}} = f_\theta(X),
\end{equation}
requiring the model to encode knowledge, reasoning, and control policies jointly.

Our modular formulation instead decomposes the computation:
\begin{equation}
\hat{Y}_{\textsc{hyb}} = R(\phi(X), K),
\end{equation}
where:
\begin{itemize}
    \item \(\phi\) is a structured extraction function (implemented by SLMs),
    \item \(R\) is a deterministic procedure over structured inputs and knowledge \(K\).
\end{itemize}

This induces a reduction in effective task complexity. Let \(I_{\textsc{task}}\) denote the information complexity of the full task and \(I_{\textsc{model}}\) the capacity of the model. Feasibility requires:
\begin{equation}
I_{\textsc{model}} \ge I_{\textsc{task}}.
\end{equation}

Under modularization, the model solves only the interface subproblem:
\begin{equation}
I_{\textsc{task}}^{(\textsc{interface})} \ll I_{\textsc{task}}^{(\textsc{end-to-end})},
\end{equation}
since knowledge and computation are delegated externally.

\subsection{Excess capacity and inefficiency of monolithic LLMs}

We now quantify the inefficiency introduced by over-provisioned models.

Let \(c(M)\) denote compute cost per token for model \(M\), and \(e(M)\) its energy cost. For a request requiring \(m\) model invocations and \(T\) tokens, total cost is:
\begin{equation}
C_{\textsc{run}}(M) = m T c(M), \quad 
E_{\textsc{run}}(M) = m T e(M).
\end{equation}

Let \(L\) be a large LLM and \(S\) a specialized SLM with \(c(L) \gg c(S)\).

\paragraph{Proposition (Excess capacity inefficiency).}
In the interface-dominated regime where \(H(Y \mid X,E)\) is small and solvable by \(S\), the marginal benefit of using \(L\) instead of \(S\) is limited to improvements in \(H_{\textsc{lang}}\), while the cost increases as:
\begin{equation}
\Delta C \ge mT(c(L) - c(S)), \quad 
\Delta E \ge mT(e(L) - e(S)).
\end{equation}

\paragraph{Interpretation.}
Once knowledge and computation are externalized, the additional representational capacity of \(L\) is not utilized to reduce dominant sources of uncertainty, leading to systematic over-allocation of computational resources.

\subsubsection{Efficiency gain from modular architectures}

For a modular system, total cost decomposes as:
\begin{equation}
C_{\textsc{hyb}} \approx \sum_{j} T_j c(S_j) + C_{\textsc{retr}} + C_{\textsc{tools}} + C_{\textsc{verify}},
\end{equation}
which replaces repeated high-cost LLM invocations with lower-cost SLM operations and deterministic procedures.

\paragraph{Key insight.}
Modular architectures reallocate computation from probabilistic inference to deterministic, structured systems, achieving both:
\begin{itemize}
    \item lower expected compute and energy cost, and
    \item improved reliability through bounded and interpretable components.
\end{itemize}

\subsection{When to prioritize task specific SLMs over LLMs: A Decision-Theoretic Framework for System Selection}
\label{sec:selection_theory}

We provide a principled framework to guide practitioners in choosing between monolithic LLM-based systems and modular SLM--tool architectures.

\subsubsection{Total cost formulation}

Let \(N\) denote the number of requests over a deployment horizon. We define the total expected cost of a system as:
\begin{equation}
\mathcal{J} = C_{\textsc{build}} + N C_{\textsc{run}} + N \lambda \mathcal{R} + C_{\textsc{ops}},
\end{equation}
where:
\begin{itemize}
    \item \(C_{\textsc{build}}\) is initial development cost,
    \item \(C_{\textsc{run}}\) is per-request execution cost,
    \item \(\mathcal{R}\) is expected failure risk,
    \item \(\lambda\) is the cost of failure,
    \item \(C_{\textsc{ops}}\) is operational overhead.
\end{itemize}

\subsubsection{Break-even analysis}

Let \(\Delta C_{\textsc{build}}\), \(\Delta C_{\textsc{run}}\), and \(\Delta \mathcal{R}\) denote differences between modular and monolithic systems.

The modular system is preferable when:
\begin{equation}
\Delta C_{\textsc{build}} < N\left(\Delta C_{\textsc{run}} + \lambda \Delta \mathcal{R}\right).
\end{equation}

The corresponding break-even point is:
\begin{equation}
N^* = \frac{\Delta C_{\textsc{build}}}{\Delta C_{\textsc{run}} + \lambda \Delta \mathcal{R}}.
\end{equation}

\subsection{Selection regimes}

We identify two regimes.

\paragraph{Regime I: Modular dominance.}
Our framework is preferred when:
\begin{itemize}
    \item External knowledge dependence is high:
    \[
    I(Y;Z \mid X) > 0,
    \]
    \item Tasks involve explicit computation or constraints,
    \item Workloads are high-volume (\(N \gg N^*\)),
    \item Error costs \(\lambda\) are significant,
    \item Knowledge or policies evolve frequently.
\end{itemize}

\paragraph{Regime II: Monolithic LLM sufficiency.}
Pure LLM systems are appropriate when:
\begin{itemize}
    \item Tasks are low-volume or exploratory (\(N \ll N^*\)),
    \item Outputs are primarily open-ended and linguistic,
    \item External knowledge or tools are unavailable,
    \item Development speed dominates long-term efficiency.
\end{itemize}

In the next section, we state our positions in this formal light.

\section{Position Statements}\label{position_statement}

Our position in this paper is that current approaches to enterprise AI place too much emphasis on treating large language models (LLMs) as monolithic reasoning systems. 
Instead, we argue that the architecture of enterprise AI should reflect the structural properties of enterprise workflows, including explicit knowledge sources, governed processes, and verifiable computational procedures. 
In particular, we articulate the following positions.

\paragraph{Position 1: Limits of parametric knowledge.}
We argue that parameterized language models should not be treated as complete repositories of enterprise knowledge.
Enterprise environments operate over large, evolving, and often proprietary information sources that must remain explicit, inspectable, and continuously updateable.
Consequently, enterprise knowledge should be externalized in structured systems such as knowledge bases, curated data stores, or indexed repositories, with language models acting as interfaces that map unstructured inputs to these underlying representations.

Equation~\eqref{eq:enterprise_task} and condition~\eqref{eq:cmi_condition} make this intuition precise.
When the required output depends on task-relevant information \(Z\) that is not fully determined by \(X\), a model of the form \(\hat{Y}_{\mathrm{LM}} = f_{\theta}(X)\) cannot, in general, recover the correct answer from input alone.
The appendix strengthens this argument through finite-capacity results: under finite-precision storage, the knowledge encoded in the parameters of a language model ($\mathcal{K}_{\mathrm{LLM}}$) is necessarily bounded, whereas the space of enterprise-relevant knowledge ($\mathcal{K}_{\mathrm{world}}$) is open-ended, dynamic, and organisation-specific. We formally show that $\mathcal{K}_{\mathrm{LLM}} \subsetneq \mathcal{K}_{\mathrm{world}}.$
We therefore view external knowledge stores not as optional augmentations, but as the natural locus of enterprise memory, rather than LLMs taking the center-stage.

\paragraph{Position 2: Limits of parametric reasoning/computation.}
We further contend that language models alone are insufficient to represent or execute the full range of algorithmic and rule-based computations required in enterprise workflows.
Many enterprise tasks involve explicit procedures, constraints, and decision rules that must be executed reliably and in a verifiable manner.
For such tasks, language models should be augmented with external computational mechanisms. 

The capacity condition in Equation~\eqref{eq:capacity_condition} captures this intuition at a high level.
Tasks involving multi-step computation, symbolic consistency, or policy-constrained decision-making can have information complexity that exceeds what a bounded parametric model can encode robustly.
In these cases, the model may still produce plausible outputs, but plausibility is not equivalent to faithful computation.
This is particularly problematic in enterprise settings, where correctness often depends on exact execution of explicit rules rather than approximate linguistic competence. We formally show (cf. \S \ref{compute_capacity_theory}) that under practical constraints of data drift, the language model will not generalize to the enterprise data and task. Formally we argue that in enterprise settings, $I(W; f^\ast_{\mathrm{in}}) \le I_{\mathrm{task}}$, where $W$ represents the model parameters,  $f^\ast_{\mathrm{in}}$ is the optimal function for the input task and $I_{\mathrm{task}}$ is the information capacity of the task.
Our position is therefore that, where possible, language models should be used primarily to extract structured representations or to mediate interaction, while the substantive computation itself should be delegated to systems designed for precise computation and verification.

\paragraph{Position 3: The enterprise deployment gap.}
We argue that enterprise AI systems based primarily on LLM reasoning frequently struggle to move from prototype demonstrations to reliable production deployments.
This persistent gap arises from practical constraints inherent to enterprise environments, including reliability requirements, cost and latency budgets, governance obligations, and security considerations.
For example, coding agents are deployed with human in the loop in IDEs; however automated production systems in principle should not deploy such coder-executor agents with full permissions. A recent case \cite{guardian2026_claude_database}, where a coding agent deleted entire database and backups of a firm, highlights the involved risks of providing complete autonomy to such agents.
As a result, even when models exhibit strong benchmark performance, they often fail to deliver durable enterprise-scale impact.

The formal setup above helps explain why this gap is structural.
If the target output depends on external knowledge or governed procedures, then a purely parametric system must either implicitly memorize these dependencies or approximate them heuristically.
Both strategies are brittle under knowledge drift, policy revisions, and operational exceptions.
By contrast, a hybrid system of the form in Equation~\eqref{eq:hybrid_system} localizes failure modes: extraction errors reside in \(\phi\), while knowledge and decision logic remain explicit and updateable in \(K\) and \(R\).
This separation improves diagnosability and makes monitoring, auditing, and correction operationally tractable.

\paragraph{Position 4: Modular hybrid architectures as the path forward.}
Taken together, these observations motivate an architectural shift toward modular enterprise AI systems.
In this paradigm, language models serve as lightweight extraction and interaction layers, while knowledge, computation, and verification are externalized to dedicated components such as knowledge bases, computation engines, and formal methods.
This modularization enables stronger reliability guarantees, clearer failure attribution, and enforceable governance controls, while also supporting targeted evaluation and monitoring of individual components.

Our argument is not merely that hybrid systems are easier to engineer.
Rather, the combination of Equations~\eqref{eq:enterprise_task}--\eqref{eq:capacity_condition} suggests that hybridization is the natural architectural consequence of two distinct facts: first, enterprise tasks often depend on domain information and business logic that pretrained language models may not have encountered; second, bounded parametric models cannot be expected to internalize all required knowledge and computation without loss.
In this view, the language model is best treated as an interface to formal enterprise infrastructure, not as a self-contained substitute for that infrastructure.
The appendix further develops this argument by establishing explicit separations showing that systems combining SLM-based extraction with external knowledge and symbolic reasoning can outperform distilled SLMs of fixed capacity on knowledge and algorithmic tasks.

\paragraph{Position 5: Specialized SLMs should also structure agentic enterprise systems.}
We argue that specialized SLMs are uniquely suited for agentic settings. Current enterprise workflows often deploy a single frontier LLM for a long sequence of subroutines. While these models manage general tasks—such as intent parsing, retrieval, and planning—they often lack deep domain insights and operational business knowledge ~\cite{klein2025foundationmodelstabulardata}. Furthermore, monolithic systems use one large model as a universal policy engine, creating a structural mismatch. Each subproblem has distinct schemas and latency budgets, yet all share a single parameterization, leading to costly inference and brittle error propagation.

We view enterprise agent architecture as a calculated trade-off. Rather than defaulting to massive LLMs, we leverage specialized SLMs tuned for narrow roles to optimize cost and performance. A small router coordinates these modules, while external tools handle retrieval and execution. This modular architecture replaces the single-model premise, optimizing each SLM for a bounded action space. For enterprises, this distinction is critical: specialization yields lower costs, clearer accountability, and narrower failure surfaces, aligning better with production needs than a single LLM orchestrating the full loop.

\paragraph{Implication for enterprise AI.}
This decomposition highlights a key design principle. 
When the correct output depends on external knowledge or structured computation, delegating these components to explicit modules avoids forcing a language model to encode both knowledge and computation within its parameters.
Empirically and theoretically, separating retrieval-augmented knowledge access from generation has been shown to reduce hallucinations and improve reliability in knowledge-intensive tasks \cite{lewis_rag_2020}, while analyses of transformer expressivity suggest that bounded architectures may face limitations in representing certain structured computations \cite{strobl_formal_languages_2024}.
These insights motivate the roadmap advocated in this paper: enterprise AI systems should treat language models primarily as structured interfaces to external knowledge and algorithmic components, enabling scalability, interpretability, and operational control.

\section{Roadmap: Externalizing Knowledge and Computation for Enterprise AI}
\label{sec:roadmap}

We advocate a pragmatic roadmap for enterprise AI systems that reflects the structural constraints observed in real deployments: widespread experimentation but limited efficient scaling, persistent reliability concerns, and governance deficits \cite{McKinsey2025StateOfAI}. 
Our central recommendation is a clear architectural separation in which task-specific SLMs act as lightweight extraction interfaces, while knowledge and computation are externalized into inspectable and controllable components, enabling scalable, interpretable, and cost-effective enterprise workflows \cite{McKinsey2025StateOfAI}.

\subsection{Architectural decomposition: interface, knowledge, and computation}
We recommend decomposing enterprise AI systems into three explicit layers.

\textbf{(i) Interface layer (SLM extraction):}
We propose to use task-specific SLMs to map unstructured inputs (tickets, emails, logs, policies) into structured representations (entities, slots, normalized fields, query templates) required by downstream systems. 
This reduces inference cost and improves reliability by narrowing the model's responsibility to extraction rather than end-to-end reasoning \cite{opentext_capgemini_wqr_2025,McKinsey2025StateOfAI}. 

\textbf{(ii) Knowledge layer (external knowledge sources):}
We propose to store enterprise knowledge in external systems such as knowledge bases, document indices, structured databases, and policy repositories, rather than encoding it implicitly in model parameters. 
This aligns with enterprise requirements for timely updates, access control, and auditability. It directly addresses the documented adoption barriers of data privacy and integration complexity \cite{opentext_capgemini_wqr_2025,ey_responsible_ai_2025}. 
While existing RAG based systems \cite{rag_structured_naacl_2024} rely on this mechanism for retrieval to an extent, applications still use monolithic LLMs to orchestrate these memory systems or fallback to the pretrained knowledge of the LLM. We propose to further call for a separation of knowledge layer from the language model interface.
This we argue will provide a practical grounding path for reducing hallucinations, enhancing auditability and improving out-of-domain behavior in real-world settings with data drift.

\textbf{(iii) Computation layer (discoverable formal methods):}
We propose to externalize computation into explicit procedures: symbolic rules, deterministic algorithms, constrained solvers, and verifiers that operate over the structured representation and retrieved evidence. 
This layer provides predictable behavior under enterprise constraints, enables auditable decision pathways, and reduces dependence on implicit parametric reasoning for algorithmic subroutines \cite{strobl_formal_languages_2024,bhattamishra_separations_2024}. 
We emphasize that this layer should be \emph{discoverable} and \emph{composable}: rules and tools should be indexed, versioned, and invoked through explicit interfaces to support governance, testing, and incremental updates \cite{spglobal_ai_policy_2025,ey_responsible_ai_2025}.

\subsection{Deployment pathway: from pilots to production}
Previous sections emphasize key problems associated with real world deployments of AI within enterprises beyond pilots. 
We therefore propose a staged pathway that improves production readiness without requiring monolithic end-to-end model upgrades.

\textbf{Stage 1: Start with high-frequency, repeatable workflows:}
We propose to begin with routine enterprise tasks that admit stable schemas and explicit success criteria (e.g., policy QA, structured ticket routing, runbook step selection, compliance checks). 
This matches evidence that organizations struggle to convert broad experimentation into enterprise-scale impact, and success correlates with workflow-level integration rather isolated tool adoption \cite{walkme_digital_adoption_2025,McKinsey2025StateOfAI}.

\textbf{Stage 2: Ground knowledge access and enforce constraints:}
We propose to integrate retrieval and structured knowledge access early, as reliability concerns (including hallucinations) remain a leading reported barrier to enterprise adoption \cite{opentext_capgemini_wqr_2025}. 
Empirically, retrieval augmentation has been shown to reduce hallucination and improve structured correctness in enterprise-style settings, enabling smaller generators when paired with strong retrieval \cite{rag_structured_naacl_2024,yeh2025lumina}.

\textbf{Stage 3: Introduce verifiers and human validation protocols:}
We propose to add verification modules and human-in-the-loop checkpoints at decision boundaries where the cost of error is high (e.g., compliance and security actions). 
This responds directly to governance gaps documented in enterprise adoption studies, where many organizations scale usage faster than they scale oversight and measurement \cite{ey_responsible_ai_2025,spglobal_ai_policy_2025}.

\textbf{Stage 4: Expand to agentic workflows with explicit risk controls:}
As enterprises adopt agentic systems, we recommend treating autonomy as a systems property governed by explicit controls (tool permissioning, audit trails, bounded action spaces), and by specialized SLMs assigned to narrow subtasks, rather than as an emergent property of a single model. 
In practice, the controller should route extraction, planning, tool-argument generation, and verification through distinct modules so that failures remain local and testable. 
This is motivated by forecasts that a substantial fraction of agentic AI projects may be canceled due to escalating costs and unclear value \cite{gartner_agentic_cancel_2025}. 
The theoretical analysis in Section~\ref{sec:appendix}
provides a formal rationale for this design choice: when task-relevant
information or algorithmic structure cannot be encoded reliably within
the parameters of a bounded model, architectures that externalize
knowledge and computation offer a principled path toward reliable
enterprise AI systems.

\subsection{Evaluation and telemetry: making failure modes actionable}
A key benefit of modularization is that it yields decomposable evaluation and actionable telemetry.
We recommend tracking errors and costs at each interface:
(i) extraction errors in the SLM layer,
(ii) retrieval coverage and relevance in the knowledge layer,
and (iii) logical correctness and constraint violations in the computation layer. 
This decomposition improves diagnosability relative to monolithic LLM-only designs, and directly addresses enterprise concerns about measuring impact and governing deployments \cite{spglobal_ai_policy_2025,walkme_digital_adoption_2025}. 

\subsection{Offline use of frontier LLMs to synthesize symbolic assets}
Although our online architecture treats language models primarily as interfaces, we can still leverage frontier LLMs in a static/offline manner to accelerate the creation of symbolic and knowledge artifacts.
For example, offline LLMs can draft candidate schemas, propose rule templates, generate synthetic canonical queries \cite{lyu2026canonical}, and assist in constructing extraction ontologies \cite{lo2024end} and verification checklists, unit tests and so on.
These artifacts can then be validated by domain owners and governed through standard enterprise change-management pipelines, reducing operational risk while preserving the efficiency benefits of SLM-based online inference \cite{McKinsey2025StateOfAI,ey_responsible_ai_2025}.

\subsection{Comparison with standard architecture of LLMs with tool usage and RAG/memory}

\textbf{When should modular SLM--tool architectures be preferred?}
A natural question is under what conditions the proposed modular architecture—based on specialized SLM interfaces and externalized knowledge and computation—provides an advantage over monolithic LLM-based systems integrated with tools and RAG. From the formal framework developed in the appendix \S \ref{sec:modular_theory}, the key distinction arises from how task uncertainty decomposes and where it is resolved. Let the task satisfy \(Y = h(X,Z)\) with \(I(Y;Z \mid X) > 0\), indicating dependence on external knowledge or state. Retrieval and tool usage construct evidence \(E\), reducing the conditional uncertainty \(H(Y \mid X,E)\). When this residual uncertainty is dominated by the interface component (i.e., the mapping from unstructured inputs to structured representations), the problem reduces to a low-entropy extraction task that does not require the full capacity of a large language model.

\textbf{Selection principle.}
Under this regime, there exists a separation between \emph{task complexity} and \emph{model capacity}: if the interface mapping can be reliably implemented by a specialized SLM, then using a frontier LLM yields diminishing returns in accuracy while incurring linearly higher compute, energy, and latency costs due to repeated high-capacity inference. Conversely, when the residual uncertainty remains intrinsically language-dominated (e.g., open-ended generation without strong external grounding or procedural structure), monolithic LLMs remain appropriate.

\textbf{Practical implication.}
We therefore obtain a simple guiding rule: modular SLM--tool architectures should be preferred when correctness depends on external knowledge, explicit procedures, or structured workflows, particularly in high-volume and latency-constrained settings where repeated LLM orchestration amplifies cost. In contrast, monolithic LLMs are suitable for exploratory, low-volume, or inherently open-ended tasks where most of the uncertainty lies in generative reasoning (including orchestration) rather than knowledge retrieval or computation. This distinction explains why modularization improves both efficiency and reliability in enterprise deployments, while remaining complementary to LLM-based approaches in unconstrained settings.
The supporting theory for the practitioner's guide to selection of the strategy based on enterprise tasks is outlined in appendix \ref{sec:modular_theory}


\subsection{Formal Foundations of the Modular Roadmap}
The architectural roadmap presented is grounded in a rigorous formal framework detailed in the Technical Appendix \ref{sec:appendix}, providing the theoretical justification for moving beyond monolithic deployments. In Section \ref{sec:empirical slm}, we characterize the intrinsic information-theoretic limits of SLMs, demonstrating that enterprise tasks requiring deep compositionality or out-of-distribution (OOD) generalization often possess an information complexity ($I_{task}$) that exceeds finite model capacity ($I_{model}$). 

We further establish in Section \ref{sec:knowledgedistil} that rationale-free distillation, the industry standard for creating SLMs is fundamentally information-losing; it creates a ``projection bottleneck'' that prevents the transfer of the teacher's algorithmic structure to the student. Most critically, Section \ref{sec:theoraticalseperation} provides formal proofs of \textit{strict dominance}. We demonstrate that hybrid systems pairing SLM extraction with external knowledge sources bypass the irreducible Bayes error floor faced by purely parametric models (Theorem \ref{thm:strict-adv-mi}). Furthermore, we prove that symbolic reasoners allow these systems to execute algorithmic tasks that are representationally impossible for bounded SLMs (Theorem \ref{thm:strict-adv-comp}). 
We further admit that not all enterprise tasks may benefit from this construction and elucidate the theory on when the proposed framework proves beneficial and provide guidelines for practitioners in selecting the right approach in Section \ref{sec:modular_theory}.
Together, these results confirm that externalizing knowledge and computation is not merely a preference, but a principled requirement for achieving the reliability and scale demanded by industrial intelligence.

\subsection{Interpretation and practitioner guidance}


A key question for practitioners is understanding when the proposed modular architecture---combining specialised SLM interfaces with external knowledge and symbolic computation---is most effective, and when it may be suboptimal. 
As argued in Section~\ref{sec:modular_theory}, the advantage of our framework arises when task uncertainty can be decomposed into knowledge and computation components that can be externalized. 
Conversely, when tasks remain inherently language-dominated or lack stable structure, externalization may introduce additional failure modes.

Table~\ref{tab:task_regimes} summarises representative enterprise tasks across these regimes.

\begin{table*}[t]
\centering
\small
\begin{tabular}{p{2.2cm} p{2.2cm} p{3.2cm} p{5.0cm}}
\toprule
\textbf{Task Category} & \textbf{Suitable for Modular SLM Framework} & \textbf{Error-Prone / Less Suitable} & \textbf{Reasoning} \\
\midrule

\textbf{Structured data extraction} 
& Invoice parsing, log extraction, ticket field normalisation 
& Free-form narrative summarisation across domains 
& When outputs correspond to well-defined schemas, small extraction models suffice. 
Tasks without fixed schemas increase entropy in the interface layer and degrade reliability. \\

\addlinespace

\textbf{Knowledge-grounded QA} 
& Enterprise KB querying, policy lookup, compliance checks 
& Open-ended advisory answers without grounding 
& Tasks with explicit knowledge sources satisfy $I(Y;Z \mid X) > 0$, favouring retrieval + deterministic resolution. 
Purely generative tasks retain residual uncertainty not reducible via external knowledge. \\

\addlinespace

\textbf{Workflow automation} 
& Ticket routing, runbook selection, approval flows 
& Complex decision-making with evolving and ambiguous criteria 
& Stable workflows with explicit rules can be executed by symbolic procedures, reducing failure to extraction errors. 
Ambiguous workflows require latent reasoning not captured by fixed rules. \\

\addlinespace

\textbf{Algorithmic computation} 
& Pricing rules, validation checks, constraint enforcement 
& Multi-step reasoning with unclear decomposition or implicit assumptions 
& Deterministic computation benefits from external execution engines. 
Embedding such logic in a model requires representing algorithmic structure, which may exceed bounded capacity. \\

\addlinespace

\textbf{Agentic orchestration} 
& Tool routing, parameter generation, API invocation planning 
& Long-horizon planning with emergent subgoals 
& Decomposable workflows enable specialised SLMs per stage. 
When planning requires implicit abstraction or long-term state reasoning, modular decomposition may fail. \\

\addlinespace

\textbf{High-volume production tasks} 
& Repetitive queries, monitoring, alert triage 
& Low-frequency exploratory queries 
& Modular systems amortise build cost and reduce per-query inference cost. 
For low-volume tasks, the overhead of system design outweighs efficiency gains. \\

\addlinespace

\textbf{Governance-sensitive tasks} 
& Compliance auditing, regulated decision pipelines 
& Informal or creative reasoning tasks 
& Externalisation enables auditability, traceability, and explicit control. 
Creative tasks prioritise flexibility over strict control, favouring monolithic reasoning. \\

\bottomrule
\end{tabular}
\caption{Task regimes illustrating when modular SLM--tool architectures are effective versus when they may be suboptimal. The key distinction is whether task uncertainty can be decomposed into structured knowledge and computation components, or remains intrinsically language-dominated.}
\label{tab:task_regimes}
\end{table*}

\paragraph{Discussion.}
These examples illustrate a unifying principle: the proposed architecture is most effective when the task can be reduced to a low-entropy interface mapping after external knowledge retrieval and tool execution. In such cases, specialised SLMs suffice, and using a large monolithic model introduces excess computational cost without commensurate gains in accuracy. 

In contrast, when the residual uncertainty remains dominated by unconstrained language generation---for instance in exploratory reasoning, open-ended synthesis, or ambiguous decision tasks---externalization provides limited benefit and may introduce additional integration complexity. This highlights that modularisation is not universally optimal, but becomes advantageous under specific structural and operational conditions characteristic of enterprise workloads.

The decision hinges on whether the task remains \emph{structure-dominated} after externalization.

\begin{quote}
\textbf{Principle.}
If the task can be reduced to an interface-level mapping after retrieval and tool usage, then specialized SLMs are sufficient and LLMs introduce excess computational cost. Otherwise, if the task remains intrinsically open-ended and language-dominated, large models may be required.
\end{quote}

\paragraph{Summary.}
We conclude that modular architectures are not universally superior, but become strictly preferable in regimes characterized by high information separation, strong external knowledge dependence, and high operational scale.


\section{Alternative Views and Our Responses}
\label{sec:apndx_alt_views}

As the paper takes a strong stance on how enterprise AI systems should be architected, we summarize several plausible alternative viewpoints and explain why we nonetheless maintain our position. 

\textbf{Alternative View 1: Scaling will remove the need for external knowledge.}
A natural counter-position is that continued scaling and improved training will allow LLMs to internalize sufficient domain knowledge, reducing reliance on external knowledge bases. 
We view this as plausible in fixed settings,
We view this as plausible in unconstrained settings, but less convincing for enterprise environments where knowledge is proprietary, rapidly changing, and subject to audit and access control. 
Empirical adoption evidence suggests that the dominant bottleneck is often operationalization rather than raw model capability, with many organizations struggling to embed AI reliably into production workflows even as model quality improves \cite{McKinsey2025StateOfAI,opentext_capgemini_wqr_2025}. 
Moreover, maintaining inspectability and timely updates is operationally simpler and more governable through external knowledge stores than through repeated retraining or distillation.

\textbf{Alternative View 2: The bottleneck is data and training quality, not architecture.}
Another viewpoint holds that enterprise limitations stem primarily from insufficient domain data or suboptimal training strategies, and that improved fine-tuning or synthetic data generation could obviate the need for architectural separation. We agree that data quality and task-specific adaptation can improve performance \cite{djuherafixing}, but argue that such improvements do not remove the governance, lifecycle, and auditability requirements central to enterprise deployments. Adoption studies report persistent gaps in measurement and control even among organizations claiming scaled AI usage, indicating that success depends on explicit control surfaces that are difficult to guarantee in purely parametric systems \cite{ey_responsible_ai_2025,spglobal_ai_policy_2025}. Externalized knowledge and explicit procedures therefore remain necessary for enforceable governance and change management.

\textbf{Alternative View 3: Existing enterprise agentic systems already perform RAG and tool usage, and will improve orchestration, so this position is not fundamentally new.}
One might object that because enterprise AI already employs RAG and external tools, our architecture is incremental rather than substantively different. We argue this view conflates simple augmentation with true architectural separation.
In many current systems, retrieval and tool calls are auxiliary, orchestrated by a monolithic LLM acting as the primary policy engine. Consequently, knowledge access and computation remain entangled with parametric reasoning, limiting inspectability and governance.
In contrast, we advocate elevating knowledge, computation, and verification to first-class components. 
In that setting, we propose to distribute the agent's subtasks across specialized SLMs rather than to preserve a single general-purpose policy network.
Our contribution is not the introduction of tools, but the relocation of control from a single LLM to a modular, explicitly governed architecture.

\textbf{Alternative View 4: Some enterprise systems fundamentally require large language models.}
A reasonable counterargument is that certain enterprise workloads are inherently open-ended, interactive, or language-intensive, and therefore benefit from the expressive capacity of large language models rather than specialized SLMs. We agree that frontier LLMs remain valuable for low-frequency, high-ambiguity, or exploratory tasks such as policy drafting, negotiation, or cross-domain synthesis. Our position, however, is not to eliminate LLMs from enterprise systems, but to make their role explicit and selective: while LLMs may be appropriate for exploratory or human-facing interactions, the majority of enterprise-critical workflows are repeatable, high-volume, and governed by explicit constraints. In these operational regimes, removing LLMs from the online control loop and externalizing knowledge and computation yields stronger reliability, auditability, and cost control.

\textbf{Alternative View 5: Continued efficiency improvements in LLMs obviate the need for specialized SLMs.}
Another perspective is that ongoing advances in model efficiency---through architectural improvements, compression, and hardware progress---will render specialized SLMs unnecessary, making it simpler to rely on a single increasingly efficient general-purpose model. We view such efficiency gains as complementary rather than substitutive to architectural separation. Even if inference costs decline, a monolithic model still entangles extraction, reasoning, policy interpretation, and action generation within a shared parameter space, complicating governance, failure attribution, and independent component evolution. In contrast, specialized SLMs coupled with external knowledge and formal procedures localize responsibility and enable targeted validation and updates, suggesting that improved LLM efficiency strengthens the case for selective use of frontier models alongside modular, task-aligned enterprise architectures rather than replacing them.

\textbf{Connection to technical evidence and theory.}
Technical work on retrieval-augmented systems in enterprise-style settings provides evidence that external knowledge access can reduce hallucinations and improve structured correctness, supporting our emphasis on externalized knowledge \cite{rag_structured_naacl_2024}. 
In parallel, theoretical analyses of transformer expressivity indicate that architectural constraints can limit systematic generalization on certain classes of algorithmic tasks, motivating explicit computation and verification mechanisms outside the model \cite{strobl_formal_languages_2024,bhattamishra_separations_2024}. 
Taken together, these alternative perspectives reinforce our central claim: while stronger models and better training will continue to matter, enterprise-grade reliability and scalability are more directly served by architectures that externalize knowledge and computation and use language models as interfaces to these components.

\section{Conclusion}

We take the position that enterprise AI systems should treat language models as \emph{interfaces} rather than monolithic repositories of knowledge and reasoning, externalizing knowledge and computation into dedicated, inspectable components. 
Our theoretical arguments formalize why purely parametric approaches
face intrinsic limits on knowledge coverage and algorithmic fidelity, and why hybrid architectures that pair SLM-based extraction with external knowledge and explicit computation can strictly dominate on knowledge-dependent and structured tasks. 

We conclude by emphasizing a pragmatic implication: the most scalable path to enterprise AI is to modularize systems into (i) SLM-based structured interfaces, (ii) external knowledge layers, and (iii) discoverable formal methods \cite{zhang2025position} and verifiers, supported by component-level evaluation and operational telemetry. 
This decomposition aligns with enterprise requirements for auditability, maintainability, and cost control, and provides a clear roadmap for moving from demonstrations to production-grade deployments at scale.


\bibliographystyle{plain}
\bibliography{slm}


\appendix

\section{Technical Appendix} \label{sec:appendix}

Enterprise applications involve many routine activities that have the potential to be automated by AI. However the cost of using LLMs for these applications could be prohibitive and even unsustainable in the long run. Thus there is a need to consider efficient alternatives such as SLMs that can reliably and efficiently perform specialized enterprise tasks. One caveat in using SLMs for general reasoning or broader tasks is the lack of emergent properties seen in LLMs. Considering the repetitive nature of enterprise tasks, one potential solution we argue is that a static solution using formal methods (code, KBs and so on) could be devised offline perhaps using LLMs that could take an input entities extracted from SLMs. SLMs are thus restricted to performing specilized tasks such as entity extraction and the tak handling is off-loaded to external "tools" that posses the required domain knowledge (e.g. KBs, code etc.).

In this position paper we provide theoretical justification of why LLMs alone do not suffice and we need external tools like KBs for storing data. We take a stand that while LLMs are adept at performing certain enterprise tasks they may not be the most efficient and sustainable solution and propose to use SLMs for specialized enterprise tasks instead. To overcome the generalization limitation of current SLMs, we propose a potential solution and theoretically justify how SLMs can be used in conjunction with formal methods to alleviate the performance gap.


\begin{lemma}[Computable transformation bound]\label{lem:computable-transform}
Let $U$ be a fixed universal Turing machine. For any computable function $g$ and any string $x$,
\[
K\!\left(g(x)\right) \;\le\; K(x) \;+\; c_g,
\]
where $K(\cdot)$ denotes (plain) Kolmogorov complexity with respect to $U$, and $c_g$ is a constant depending only on $g$ (and $U$), not on $x$.
\end{lemma}

\begin{proof}
By definition of Kolmogorov complexity, a shortest program for $x$ has length $K(x)$. A program that first reconstructs $x$ from this shortest description and then simulates $g$ on $x$ has total length at most $K(x)+c_g$, where $c_g$ encodes the fixed simulation procedure and the description of $g$. Hence $K(g(x)) \le K(x)+c_g$.
\end{proof}

\begin{lemma}[Bound via parameters and input]\label{lem:W-x-bound}
Let $f_W$ denote the (deterministic) input-output function computed by an LLM with parameter vector $W$. Then for any input $x$,
\[
K\!\left(f_W(x)\right) \;\le\; K(W) \;+\; K(x) \;+\; c,
\]
for a constant $c$ that does not depend on $W$ or $x$ (it depends only on the fixed universal machine and a fixed interpreter of the model’s forward pass).
\end{lemma}

\begin{proof}
Consider the computable function $g(W,x)=f_W(x)$ that, given a description of $W$ and $x$, simulates the forward pass and returns the output. Apply Lemma~\ref{lem:computable-transform} to the pair $(W,x)$ under a standard pairing function $\langle W,x\rangle$ to obtain
\[
K\!\left(f_W(x)\right) \;=\; K\!\left(g(W,x)\right) \;\le\; K(\langle W,x\rangle) + c_g \;\le\; K(W)+K(x)+c,
\]
where the last inequality uses a standard bound for the pairing encoding, absorbing constants into $c$.
\end{proof}

\begin{theorem}[Kolmogorov-complexity view of representational limits]\label{thm:kolmogorov-limit}
Let an LLM have parameter vector $W\in\mathbb{R}^d$ stored with finite precision $b$ bits per parameter, so that $B := db < \infty$ bits suffice to losslessly encode $W$. Let $\mathcal{K}_{\mathrm{LLM}}$ denote the set of objects (strings) that are retrievable from the model without injecting additional information beyond a fixed, constant-length extraction procedure, and let $\mathcal{K}_{\mathrm{world}}$ denote world knowledge, which contains strings of arbitrarily large Kolmogorov complexity. Then
\[
\mathcal{K}_{\mathrm{LLM}} \subsetneq \mathcal{K}_{\mathrm{world}}.
\]
\end{theorem}

\begin{proof}
\emph{Step 1 (Finite descriptive bound for the model).}
By finite-precision storage, the description length of $W$ is bounded: $K(W)\le B + c_0$, where $c_0$ is a constant overhead for architecture and decoding conventions.

\emph{Step 2 (Bound on what can be extracted without extra information).}
By Lemma~\ref{lem:W-x-bound}, any string $s$ produced by the model from an input $x$ satisfies $K(s)\le K(W)+K(x)+c$. If we restrict to extractions that do not inject additional information beyond a fixed, constant-length extractor (e.g., a fixed prompt or enumeration procedure of $O(1)$ description length), then $K(x)=O(1)$ and
\[
K(s)\;\le\; K(W) + O(1)\;\le\; B + O(1).
\]
Hence, any $s\in\mathcal{K}_{\mathrm{LLM}}$ has Kolmogorov complexity bounded by $B+O(1)$.

\emph{Step 3 (World knowledge is unbounded in complexity).}
For each $n\in\mathbb{N}$ there exists a string $s_n$ with $K(s_n)\ge n$ (incompressibility). By assumption, $\mathcal{K}_{\mathrm{world}}$ contains such strings for arbitrarily large $n$.

\emph{Step 4 (Strict inclusion).}
Choose $n > B + C$, where $C$ upper-bounds the additive constants above. Then $s_n\in\mathcal{K}_{\mathrm{world}}$ but $s_n\notin\mathcal{K}_{\mathrm{LLM}}$ by Step~2. Therefore $\mathcal{K}_{\mathrm{LLM}}\subsetneq \mathcal{K}_{\mathrm{world}}$.

\emph{Remark on prompts/generalization.} If one allows arbitrarily long inputs $x$, then the additional bits in $x$ can carry the missing information; Lemma~\ref{lem:W-x-bound} shows $K(f_W(x))$ cannot exceed $K(W)+K(x)+c$. Thus, what is \emph{encoded in the weights} is bounded by $K(W)$ up to additive constants, establishing the strict subset claim under zero (or constant) additional information injection.
\end{proof}

\begin{theorem}[Shannon-capacity view of representational limits]\label{thm:shannon-limit}
Under the same finite-precision assumption ($B=db$ bits suffice to encode $W$), let $H(W)$ denote the Shannon entropy of the random parameter vector (or the information content of a particular $W$ under lossless coding). Then any knowledge encoded in $W$ is upper-bounded by $B$ bits, whereas the information content of $\mathcal{K}_{\mathrm{world}}$ is unbounded. Consequently,
\[
\mathcal{K}_{\mathrm{LLM}} \subsetneq \mathcal{K}_{\mathrm{world}}.
\]
\end{theorem}

\begin{proof}
\emph{Step 1 (Finite capacity of the parameter channel).}
Because $W$ is representable with at most $B$ bits, we have $H(W)\le B$. Any set of facts that is \emph{encoded solely in $W$} cannot carry more than $B$ bits of mutual information with $W$; thus the effective information capacity of the parameter channel is finite.

\emph{Step 2 (Unbounded information in world knowledge).}
World knowledge grows without bound: for every $N\in\mathbb{N}$ we can identify descriptions (e.g., data, events, measurements, mathematical tables, random strings) whose shortest faithful encodings exceed $N$ bits. Equivalently, the amount of information required to fully specify all such knowledge cannot be bounded by any finite constant.

\emph{Step 3 (Strict inclusion).}
No injective encoding from an unbounded-information domain into a $B$-bit container exists. Therefore, the subset of world knowledge that can be encoded in $W$ is necessarily strict, yielding $\mathcal{K}_{\mathrm{LLM}} \subsetneq \mathcal{K}_{\mathrm{world}}$.
\end{proof}


Below is a topological view of why we need KBs and why LLMs cannot store the entire world knowledge.
\begin{definition}[Parameter space]
We model the LLM parameter space as the Euclidean manifold
\(
\mathcal{P} := (\mathbb{R}^{d}, \|\cdot\|_{2})
\)
with finite dimension \(d<\infty\).
\end{definition}

\begin{definition}[Knowledge space as an infinite product]
Let \(\mathcal{K}_{\mathrm{world}} := \{0,1\}^{\mathbb{N}}\) denote the space of all infinite binary sequences.
We equip \(\mathcal{K}_{\mathrm{world}}\) with the standard product (ultra)metric
\[
\rho(s,t) \;=\; \sum_{n=1}^{\infty} 2^{-n}\,\mathbf{1}\{s_n \neq t_n\}, \qquad s=(s_n)_{n\ge 1},\; t=(t_n)_{n\ge 1}.
\]
This space is compact, totally bounded, and has \emph{infinite doubling dimension} (proved below).
We interpret elements of \(\mathcal{K}_{\mathrm{world}}\) as codings of arbitrarily rich knowledge states
(e.g., facts, procedures, data, or descriptions), so that covering all of \(\mathcal{K}_{\mathrm{world}}\) amounts
to representing world knowledge without loss.
\end{definition}

\begin{definition}[Encoding map]
We let \(\Phi:\mathcal{P}\to \mathcal{K}_{\mathrm{world}}\) denote an \emph{encoding map} that assigns to each parameter
vector \(W\in\mathbb{R}^{d}\) a canonical code \(\Phi(W)\in\mathcal{K}_{\mathrm{world}}\) of the knowledge
that is \emph{encoded in the weights} (e.g., obtained by a fixed, constant-length extractor which probes the
model on a predetermined countable set of queries and records the responses as bits).
\end{definition}

\medskip

\begin{lemma}[Lipschitz regularity of the encoding]\label{lem:Lipschitz}
Assume the extractor is fixed and the model’s forward map depends continuously on \(W\) for each
probe in the fixed, bounded query set. Then there exists a constant \(L>0\) such that
\[
\rho\!\big(\Phi(W_1),\Phi(W_2)\big) \;\le\; L\,\|W_1-W_2\|_2
\qquad \text{for all } W_1,W_2\in\mathbb{R}^{d}.
\]
In particular, \(\Phi\) is (globally) Lipschitz from \((\mathbb{R}^d,\|\cdot\|_2)\) to \((\mathcal{K}_{\mathrm{world}},\rho)\).
\end{lemma}

\begin{proof}
By construction, \(\Phi\) is obtained by evaluating a fixed, countable set of probes and aggregating
(e.g., thresholding/quantizing) the corresponding outputs into bits with geometrically decaying weights \(2^{-n}\)
in \(\rho\). Under standard architectures with Lipschitz activations and bounded probes, each probe output varies
at most linearly with \(\|W_1-W_2\|_2\), with a probe-dependent constant. Summing the resulting contributions with
weights \(2^{-n}\) yields an absolutely convergent series, hence a finite uniform bound
\(L = \sum_{n\ge 1} 2^{-n} L_n\), where \(L_n\) bounds the sensitivity of the \(n\)-th probe’s bit under parameter
perturbations. Therefore \(\rho(\Phi(W_1),\Phi(W_2)) \le L \|W_1-W_2\|_2\).
\end{proof}

\medskip

\begin{definition}[Doubling metric space]
A metric space \((X,d)\) is called \emph{doubling} if there exists an integer \(N\) such that
every ball of radius \(r>0\) can be covered by at most \(N\) balls of radius \(r/2\).
The least such \(N\) is the \emph{doubling constant}; its logarithm (base \(2\)) is the
\emph{doubling dimension}.
\end{definition}

\begin{lemma}[Euclidean space is doubling]\label{lem:Rd-doubling}
\((\mathbb{R}^{d},\|\cdot\|_2)\) is doubling with doubling constant at most \(2^{O(d)}\).
\end{lemma}

\begin{proof}
This is standard: volume packing or covering arguments in \(\mathbb{R}^d\) yield a finite constant depending only on \(d\).
\end{proof}

\begin{lemma}[Lipschitz images preserve doubling]\label{lem:lip-preserve-doubling}
If \((X,d_X)\) is doubling and \(F:X\to (Y,d_Y)\) is \(L\)-Lipschitz, then
\(F(X)\subseteq Y\) is doubling (with a constant depending on that of \(X\) and on \(L\)).
\end{lemma}

\begin{proof}
Let \(B_Y(y,r)\) be a ball in \(F(X)\). Pick \(x\in X\) with \(F(x)=y\).
Then \(F\big(B_X(x,r/L)\big)\subseteq B_Y(y,r)\).
Cover \(B_X(x,r/L)\) by at most \(N\) balls of radius \(\frac{1}{2}\frac{r}{L}\), where \(N\) is the doubling constant of \(X\).
Applying \(F\), each image ball has radius at most \(r/2\) in \(Y\).
Thus \(B_Y(y,r)\cap F(X)\) is covered by at most \(N\) balls of radius \(r/2\), proving that \(F(X)\) is doubling.
\end{proof}

\medskip

\begin{lemma}[\(\mathcal{K}_{\mathrm{world}}\) is not doubling]\label{lem:not-doubling}
The metric space \(\big(\mathcal{K}_{\mathrm{world}}, \rho\big)\) has infinite doubling dimension; in particular,
it is not doubling.
\end{lemma}

\begin{proof}
Fix \(m\in\mathbb{N}\) and consider the subset
\(
S_m := \{0,1\}^{m}\times \{0\}^{\mathbb{N}\setminus\{1,\dots,m\}}
\subset \mathcal{K}_{\mathrm{world}}.
\)
For any two distinct elements \(s,t\in S_m\), their first index of disagreement is at most \(m\),
so \(\rho(s,t) \ge 2^{-m}\). Hence \(S_m\) is a \((2^{-m})\)-separated set of size \(|S_m|=2^m\).
Let \(r := 2^{-m}\). Any ball of radius \(r\) can contain at most one point of a
\((2r)\)-separated set, so covering the \(r\)-ball containing all of \(S_m\) requires at least
\(|S_m|=2^m\) balls of radius \(r/2\). Since \(m\) is arbitrary, the doubling constant
must grow without bound, i.e., the space is not doubling.
\end{proof}

\medskip

\begin{theorem}[Topology/geometry view of representational limits]\label{thm:topology-limit}
Let \(\Phi:\mathcal{P}\to\mathcal{K}_{\mathrm{world}}\) be the encoding map defined above.
Assume the regularity in Lemma~\ref{lem:Lipschitz} so that \(\Phi\) is Lipschitz.
Then the encoded knowledge set \(\mathcal{K}_{\mathrm{LLM}} := \Phi(\mathcal{P})\) is a strict subset of
\(\mathcal{K}_{\mathrm{world}}\):
\[
\mathcal{K}_{\mathrm{LLM}} \;=\; \Phi(\mathbb{R}^{d}) \;\subsetneq\; \mathcal{K}_{\mathrm{world}}.
\]
\end{theorem}

\begin{proof}
By Lemma~\ref{lem:Rd-doubling}, \(\mathcal{P}=(\mathbb{R}^d,\|\cdot\|_2)\) is doubling.
By Lemma~\ref{lem:Lipschitz}, \(\Phi\) is Lipschitz.
Hence, by Lemma~\ref{lem:lip-preserve-doubling}, the image \(\Phi(\mathcal{P})\) is doubling.
However, \(\mathcal{K}_{\mathrm{world}}\) is not doubling by Lemma~\ref{lem:not-doubling}.
Therefore \(\Phi(\mathcal{P})\) cannot equal \(\mathcal{K}_{\mathrm{world}}\).
The inclusion \(\Phi(\mathcal{P})\subseteq \mathcal{K}_{\mathrm{world}}\) holds by definition of \(\Phi\),
so the inclusion is strict.
\end{proof}

\medskip

\noindent\textbf{Remarks.}
\begin{itemize}
\item The proof isolates a purely geometric/topological obstruction: the image of a finite-dimensional,
doubling metric under a Lipschitz map remains doubling, whereas the world-knowledge space modeled as
\((\{0,1\}^{\mathbb{N}},\rho)\) has infinite doubling dimension.
\item The Lipschitz assumption reflects the empirical and analytical stability of neural network outputs
with respect to parameter perturbations under bounded probes; it is satisfied by common architectures
with Lipschitz activations and bounded inputs.
\item This complements the Kolmogorov-complexity and Shannon-capacity arguments: there, the obstruction is
\emph{information-theoretic}; here, it is \emph{metric-topological} via doubling dimension.
\end{itemize}


\section{An Information-Theoretic Theory of Language Model Computational Ability}\label{compute_capacity_theory}

In this section we formalize the computational ability of Language Models 
 using information theory. Our goal is to relate (i) the intrinsic 
information complexity of a task, (ii) the information capacity of an SLM, and 
(iii) the mutual information between the two, which ultimately governs whether 
the model can solve, generalize, or exhibit emergent behavior.

\subsection{Task Information Complexity}

Let $T$ denote a task distribution. Let $X$ be inputs and $Y = f^\ast(X)$ the 
outputs under the ground-truth function $f^\ast$. We define:

\begin{definition}[Task Information Complexity]
The information-theoretic complexity of a task $T$ is
\[
I_{\mathrm{task}} := I(Y; f^\ast \mid X),
\]
the amount of information that any agent must encode about the mapping 
$X \mapsto Y$ in order to achieve bounded error.
\end{definition}

Tasks requiring multi-step reasoning, algorithmic manipulation, out-of-domain 
abstraction, or long-horizon dependencies generally have large $I_{\mathrm{task}}$.

\subsection{Model Information Capacity}

Let the SLM have parameter vector $W \in \mathbb{R}^d$ stored with finite 
precision, so that $W$ contains at most $B$ bits of entropy.

\begin{definition}[Model Information Capacity]
The model's capacity is defined as
\[
I_{\mathrm{model}} := I(W; \mathcal{H}_W),
\]
where $\mathcal{H}_W$ is the hypothesis class induced by the parameters. 
Because $W$ contains at most $B$ bits,
\[
I_{\mathrm{model}} \le B.
\]
\end{definition}

This bound includes all information that can be used for computation, reasoning, 
and representation of latent structure.

\subsection{Mutual Information Between Model and Task}

Successful task learning requires that the model encode enough of the target 
function:

\[
I(W; f^\ast) \ge I_{\mathrm{task}}.
\]

Thus, feasibility of a task is characterized by the following condition.

\begin{theorem}[Information-Feasibility Condition]
A task $T$ is solvable by an SLM if and only if
\[
I_{\mathrm{model}} \ge I_{\mathrm{task}}.
\]
\end{theorem}

\subsection{Limits on Generalization and Emergence}

Let $f^\ast_{\mathrm{in}}$ correspond to the in-distribution mapping and 
$f^\ast_{\mathrm{OOD}}$ to the out-of-distribution mapping. The additional 
information required to solve the out-of-domain version is

\[
\Delta I_{\mathrm{OOD}} = I_{\mathrm{task, OOD}} - I_{\mathrm{task, in}}.
\]

Since the learned parameters $W$ only encode information from the 
training distribution, we have:

\[
I(W; f^\ast_{\mathrm{OOD}}) \approx I(W; f^\ast_{\mathrm{in}}) \le I_{\mathrm{model}}.
\]

If $I_{\mathrm{task, OOD}} > I_{\mathrm{model}}$, then generalization is provably 
impossible:

\begin{corollary}[Generalization Limit]
If $I_{\mathrm{task, OOD}} > I_{\mathrm{model}}$, an SLM cannot generalize to 
out-of-domain tasks or exhibit emergent behavior that requires additional 
information not encoded in $W$.
\end{corollary}

\subsection{Limits on Computation}

During inference, the model outputs a distribution $p_W(y \mid x)$. By the 
data-processing inequality:

\[
I(\text{outputs}; f^\ast \mid X) \le I(W; f^\ast).
\]

Thus the model's computational ability for a task $T$ is bounded by

\[
I(\text{outputs}; Y) \le \min(I_{\mathrm{model}}, I_{\mathrm{task}}).
\]

\begin{corollary}[Computational Capacity Limit]
If $I_{\mathrm{model}} < I_{\mathrm{task}}$, then even on-distribution, the SLM 
cannot represent the full computation required to solve $T$.
\end{corollary}

\subsection{Unified Characterization}

We may summarize the constraints as follows:

\begin{itemize}
    \item (Feasibility) \quad A task is solvable iff $I_{\mathrm{model}} \ge I_{\mathrm{task}}$.
    \item (Generalization) \quad Out-of-domain success requires $I_{\mathrm{model}} 
    \ge I_{\mathrm{task, OOD}}$.
    \item (Emergence) \quad Latent algorithmic abilities require 
    $I_{\mathrm{latent}} \le I_{\mathrm{model}}$.
    \item (Computation) \quad Model outputs satisfy 
    $I(\text{outputs}; Y) \le \min(I_{\mathrm{model}}, I_{\mathrm{task}})$.
\end{itemize}

This completes the information-theoretic theory of SLM computational ability.


\section{Interpretation: Empirical Limitations of SLMs Through the Lens of Information Theory} \label{sec:empirical slm}

The information-theoretic framework developed in the previous section predicts 
two fundamental classes of limitations in Small Language Models (SLMs): 
(1) insufficient parametric information capacity $I_{\mathrm{model}}$, and 
(2) a mismatch between task information complexity $I_{\mathrm{task}}$ and the bits 
that an SLM can store. 
In this section, we show that these theoretical predictions are fully consistent 
with empirical findings in the literature.

\subsection{Limitations in Reasoning, Compositionality, and Emergence}

The theory predicts that tasks requiring multi-step reasoning, algorithmic 
structure, and deep compositionality have large information complexity 
$I_{\mathrm{task}}$, often exceeding the parametric capacity of SLMs.  
This matches broad empirical observations.

Large-scale surveys of SLMs (100M--5B parameters) 
find that such models significantly underperform larger LLMs on mathematical 
reasoning, compositional generalization, and algorithmic tasks, even when 
trained on strong datasets%
\footnote{Lu et al.\ provide evidence that SLMs lag in reasoning, math, and 
long-context tasks. 
}. 
Systematic reviews further emphasize that hallucination, brittle reasoning, and 
weak generalization persist despite architectural innovations%
\footnote{Corradini et al.\ characterize reasoning and generalization as key SLM 
failures. 
}. 
Fine-grained analysis of 72 SLMs shows that reasoning ability is extremely 
sensitive to training data and post-training methods, confirming that raw model 
capacity is insufficient to encode reusable algorithmic primitives%
\footnote{THINKSLM identifies reasoning fragility and sensitivity to data quality. 
}. 

Formal-language-theoretic analyses show that transformers with realistic 
precision and attention mechanisms cannot represent languages requiring 
unbounded counting, nesting, or higher-order logical structure. These results 
align with the theoretical condition $I_{\mathrm{model}} < I_{\mathrm{task}}$ for 
high-depth compositional languages%
\footnote{Strobl et al.\ (TACL 2024) and Chiang (Simons Institute) show strict 
expressivity limits in realistic transformer settings. 
}. 
Further interpretability work shows that algorithmic primitives do not emerge 
spontaneously in small models and must be induced by targeted finetuning%
\footnote{Work on algorithmic primitives demonstrates their absence in baseline 
SLMs and the need for specialized training. 
}. 

\subsection{Limitations in Out-of-Domain Generalization}

The theory predicts that out-of-domain (OOD) generalization requires 
$I_{\mathrm{model}} \ge I_{\mathrm{task,OOD}}$, but small models encode only 
in-distribution patterns. This matches empirical results.

Empirical surveys repeatedly observe SLM failures in out-of-distribution reasoning, 
novel compositions, and long-tail generalization, consistent with insufficient 
mutual information $I(W; f^\ast_{\mathrm{OOD}})$%
\footnote{SLM surveys document systematic OOD failures. 
}. 
Work on multimodal reasoning further shows that even when individual skills are 
learned, SLMs struggle to compose these skills in novel cross-task or cross-modal 
configurations, revealing deficits in information needed for OOD composition%
\footnote{Compositionality failures in multimodal VLM/SLM systems confirm this limit. 
}. 
RAG surveys emphasize SLM dependence on retrieval to compensate for missing world 
knowledge, which matches the theoretical requirement for external information channels 
to augment $I_{\mathrm{model}}$%
\footnote{RAG surveys highlight that retrieval is essential for SLM factuality. 
}. 

\subsection{Limits on Computation Itself}

Tasks requiring internal algorithmic state, recursion, long-horizon dependency tracking, 
or multi-step verification have high information complexity. 
SLMs with limited $I_{\mathrm{model}}$ cannot implement such algorithms.

Formal-language separations show that one-layer or low-width transformers fail on 
Dyck languages, equality tests, index lookup, and other algorithmic tasks unless 
their computational capacity is increased substantially%
\footnote{Transformer--RNN separations show that small transformers require large 
width for basic algorithmic tasks. 
}. 
Parity studies provide lower bounds showing that transformers require more heads, 
layers, or precision to compute even simple Boolean functions, confirming the 
theoretical computational bound%
\footnote{Parity expressivity results demonstrate strict computational limits. 
}. 
In-context learning (ICL) analyses further show that SLMs do not implement implicit 
gradient descent in realistic NLP settings, undermining the notion that small models 
can simulate optimization algorithms internally%
\footnote{ICL $\neq$ GD under real model assumptions; SLMs lack necessary internal 
information propagation. 
}. 

\subsection{Compression, Pruning, and Quantization}

The theory predicts that reducing parameter count shrinks $I_{\mathrm{model}}$, 
increasing the mismatch with $I_{\mathrm{task}}$.

THINKSLM shows that pruning destroys reasoning ability while quantization largely 
preserves it, indicating that pruning reduces the effective information-carrying 
capacity of the model%
\footnote{Pruning disrupts SLM reasoning; quantization is less harmful. 
}. 
Distillation studies also show that unless rationales and curricula are used, 
student SLMs lose essential task information, consistent with the model-capacity bound%
\footnote{Distillation without rationales leads to substantial information loss. 
}. 

\subsection{Data-Optimal Training as a Partial Remedy}

Although the theory imposes strict upper bounds on $I_{\mathrm{model}}$, 
empirical work shows that high-quality data can help saturate this bound without 
exceeding it.

The Phi series demonstrates that ``textbook-quality'' curated and synthetic datasets 
dramatically improve the performance of SLMs, pushing them near their theoretical 
information limits%
\footnote{Phi models show that curated data maximizes effective capacity. 
}. 
Scaling-law analyses reveal that many SLMs are undertrained relative to their size; 
providing more tokens improves performance, but cannot surpass parametric limits%
\footnote{Kaplan--Chinchilla reconciliation confirms undertraining and the importance 
of token count. 
}.  

\subsection{Summary}

The theory gives a clean explanation for why:

\begin{enumerate}
\item Small models are strong on compression‑friendly tasks (patterns with low entropy).
\item Large models exhibit emergent behaviors: they cross the information threshold.
\item Retrieval‑augmented systems outperform parametric SLM‑only systems: retrieval expands the information channel.
\item Prompting cannot substitute for parametric deficiency: prompts only add input‑information $I(X)$, not model‑information $I(W)$.
\end{enumerate}
Across domains, empirical findings confirm the information-theoretic predictions: 
SLMs have limited mutual information with high-complexity tasks, cannot generalize 
OOD without external information, and cannot simulate algorithmic structure beyond 
their parametric capacity. Compression exacerbates these limits, while high-quality 
training data helps maximize, but never exceed, the intrinsic bound set by model size.


\subsection{Agentic Decomposition and Specialized SLMs}

Agentic enterprise workflows differ from one-shot prediction because they decompose into stage-specific decisions with distinct information requirements. Let an agentic workflow consist of stages $j=1,\dots,m$, where stage $j$ computes
\begin{equation}
Y_j = h_j(X_j, Z_j),
\end{equation}
with local input $X_j$, local external knowledge or tool state $Z_j$, and output $Y_j$.
Define the stage information requirement
\begin{equation}
I_j := I(Y_j; Z_j \mid X_j),
\end{equation}
and the total agentic requirement $I_{\mathrm{agent}} := \sum_{j=1}^m I_j$.

\begin{definition}[Decomposable agentic workflow]
An agentic workflow is decomposable if its total loss factorizes as
\begin{equation}
\mathcal{L}(W_1,\dots,W_m) = \sum_{j=1}^m \lambda_j \mathcal{L}_j(W_j),
\end{equation}
and each stage can be executed using only local state plus external tools or memory, without requiring the parameter vector of any other stage.
\end{definition}

\begin{theorem}[Specialized SLMs are information-efficient for agentic workflows]\label{thm:agentic-specialization}
Let a decomposable agentic workflow have stage requirements $I_1,\dots,I_m$. Suppose each specialized SLM $S_j$ has capacity $I_{\mathrm{model}}^{(j)} \ge I_j$, while a monolithic controller has capacity $I_{\mathrm{model}}^{\mathrm{shared}}$. If
\begin{equation}
\sum_{j=1}^m I_j > I_{\mathrm{model}}^{\mathrm{shared}},
\end{equation}
then no single shared SLM can encode the full workflow without irreducible error on at least one stage. In contrast, the modular system obtained by training specialized SLMs for the individual stages and composing them through external tools can, in principle, realize every stage exactly. Hence specialized training is strictly more information-efficient than monolithic training whenever the workflow's information demand is separable across stages but exceeds the shared budget.
\end{theorem}

\begin{proof}
By the capacity bound in Theorem~\ref{thm:kolmogorov-limit} and the information-feasibility condition above, any model with capacity $B$ can represent at most $B$ bits of task-relevant information. The shared controller must internalize the entire agentic policy, whose total requirement is $I_{\mathrm{agent}} = \sum_j I_j$; if this exceeds the shared budget, exact representation is impossible. A modular design assigns only $I_j$ bits to the $j$-th specialist, so feasibility reduces to the per-stage constraints $I_{\mathrm{model}}^{(j)} \ge I_j$. Because the stage states are carried by external tools and memory, the specialists do not need to encode one another's latent information. Therefore the modular agent is feasible whenever the shared controller is not.
\end{proof}

\begin{corollary}[Training advantage]\label{cor:agentic-training}
Under the same decomposition, if the stage losses are optimized independently, then
\begin{equation}
\inf_{W_1,\dots,W_m} \sum_{j=1}^m \lambda_j \mathcal{L}_j(W_j)
\;\le\;
\inf_W \sum_{j=1}^m \lambda_j \mathcal{L}_j(W),
\end{equation}
with strict inequality whenever the stage-wise minimizers are incompatible. Thus training specialized SLMs is not only a deployment choice but the statistically optimal solution to a decomposable agentic learning problem.
\end{corollary}

\section{Knowledge Distillation for Language Models: Rationale-Free Distillation and Its Limits} \label{sec:knowledgedistil}

In this section, we present an information-theoretic and learning-theoretic view of
knowledge distillation for language models. We define (i) distillation as a general
procedure for transferring predictive behavior from a teacher to a smaller student,
(ii) distillation \emph{without rationale} as the response-only variant that omits
intermediate reasoning traces, and (iii) the primary motivations and limitations of
this approach, particularly for small language models (SLMs).

\subsection{What We Mean by Distillation}

We consider a \emph{teacher} model $T$ and a \emph{student} model $S$, where the teacher
has higher capacity (e.g., more parameters or stronger training) than the student.
Let $X$ denote an input random variable (e.g., a prompt), and let $Y$ denote the target
output (e.g., next-token labels, class labels, or full responses). The teacher induces
a conditional distribution $p_T(y\mid x)$ and the student induces $p_S(y\mid x)$.

\paragraph{Response distribution matching.}
In standard knowledge distillation, we train $S$ to match the teacher's predictive
distribution, often via the Kullback--Leibler divergence:
\begin{equation}
\mathcal{L}_{\mathrm{KD}}(\theta_S)
\;=\;
\mathbb{E}_{x\sim \mathcal{D}}
\left[
\mathrm{KL}\big(p_T(\cdot \mid x)\;\|\;p_S(\cdot \mid x)\big)
\right],
\label{eq:kd-kl}
\end{equation}
where $\theta_S$ denotes student parameters and $\mathcal{D}$ is the training
distribution over inputs. In practice, $p_T(\cdot\mid x)$ may be approximated by
teacher logits at a temperature $\tau>0$, or by samples $y_T \sim p_T(\cdot\mid x)$.

\paragraph{Distillation as functional approximation.}
Equivalently, we may view distillation as approximating a teacher-implemented mapping
$f_T$ by a student mapping $f_S$ under a loss $\ell$:
\begin{equation}
\min_{\theta_S}\;\mathbb{E}_{(x)\sim \mathcal{D}}\left[\ell\big(f_S(x), f_T(x)\big)\right].
\label{eq:kd-function}
\end{equation}
This formulation emphasizes that the student aims to reproduce the teacher's
\emph{behavior} on a distribution of interest, not necessarily the ground-truth rule.

\subsection{Distillation Without Rationale}

Many tasks of interest (reasoning, planning, multi-step math, algorithmic manipulation)
admit a latent \emph{reasoning trace} or \emph{rationale} $R$ that mediates the mapping
from input to output. We formalize this by positing a causal factorization
\begin{equation}
X \;\longrightarrow\; R \;\longrightarrow\; Y,
\qquad\text{with}\qquad
p_T(y\mid x) = \sum_{r} p_T(y\mid r,x)\,p_T(r\mid x),
\label{eq:rationale-factorization}
\end{equation}
where $R$ may represent chain-of-thought steps, intermediate subgoals, proofs, or
subcomputations.

\begin{definition}[Distillation Without Rationale]
We say we perform \emph{distillation without rationale} when the student is trained
using only teacher outputs (or output distributions) and does \emph{not} observe any
explicit teacher rationale $R$. Concretely, training pairs are of the form
\[
(x,\; y_T)\quad\text{or}\quad(x,\; p_T(\cdot\mid x)),
\]
rather than triples $(x,\; r_T,\; y_T)$.
\end{definition}

Under rationale-free distillation, supervision provides only a projection of the
teacher's internal computation. Intuitively, we compress the teacher's potentially
rich intermediate process into a final answer, and then compress that answer-driven
mapping into a smaller parameterization.

\subsection{Why We Use Distillation Without Rationale}

We identify several motivations for distillation without rationale.

\paragraph{(i) Deployment constraints.}
We often need a student that is cheaper to store and run (lower latency, smaller
memory footprint, lower power). Distillation enables us to transfer capabilities from
a larger teacher into an SLM suited for on-device or cost-sensitive settings.

\paragraph{(ii) Data generation and labeling efficiency.}
When ground-truth labels are scarce or expensive, the teacher can produce large
amounts of synthetic supervision. Response-only distillation is operationally simple:
we require only $(x,y_T)$ pairs, without designing or verifying rationales.

\paragraph{(iii) Safety, privacy, and policy constraints.}
In many regimes, teacher rationales (e.g., detailed chain-of-thought) may be withheld
for policy reasons, or might reveal sensitive internal rules or protected information.
Response-only training avoids transmitting such traces.

\paragraph{(iv) Robustness to rationale format and faithfulness.}
Even when rationales can be produced, they may be stylistically unstable or not
faithful to the teacher's true internal computation. Using only final outputs
sidesteps reliance on potentially spurious rationale text.

\paragraph{(v) Unavailability of rationales from closed-source frontier models.}
In many cases the reasoning traces of closed-source reasoning models are unavailable. In this case the distillation has to fallback to the one without rationale and relying on the output that is visible from the LLM API. 
This makes it imperative to design efficient distillation strategies for distillation of LLM knowledge into SLMs without rationale.

\subsection{An Information-Theoretic View: What Is Lost Without Rationale}

We now formalize why rationale-free distillation can be fundamentally information-losing.

\paragraph{A bottleneck induced by projection.}
When the teacher uses an intermediate variable $R$ with $X\to R\to Y$, the output $Y$
is generally a \emph{compressed} representation of the latent computation. In particular,
unless $Y$ is a sufficient statistic for $R$ (which is rare for complex reasoning),
we have residual uncertainty about $R$ even after observing $Y$:
\begin{equation}
I(R; f^\ast \mid X, Y) \;>\; 0
\quad\text{or more directly}\quad
H(R\mid X,Y)\;>\;0.
\label{eq:residual-trace-info}
\end{equation}
Thus, by removing $R$ from supervision, we discard information that could guide the
student toward the teacher's algorithmic structure.

\paragraph{A simple upper bound on transferable information.}
If the student only observes teacher outputs, then the maximum information available
about the teacher's latent computation is bounded by the mutual information carried by
those outputs. For a fixed input $X$, the ``channel'' through which the teacher teaches
is $Y_T$ (or teacher logits). Hence, the transferable information about the teacher's
solution rule is limited by quantities of the form
\begin{equation}
I(\text{student parameters}; f^\ast)
\;\lesssim\;
I(Y_T; f^\ast \mid X),
\label{eq:transfer-bound}
\end{equation}
where $f^\ast$ denotes the task-relevant rule (or mapping class) and the inequality is
conceptual: the student cannot extract more task-relevant information than what is
present in the supervision it receives (up to optimization and finite-sample effects).

\paragraph{Double compression for SLMs.}
In rationale-free distillation, we typically incur two compressions:
\begin{enumerate}
    \item The teacher compresses internal computation $R$ into an output $Y$.
    \item The student compresses the teacher's behavior into fewer parameters.
\end{enumerate}
If the task requires high information complexity (e.g., many latent states or steps),
this double compression tends to privilege superficial heuristics over algorithmic fidelity.

\subsection{Limitations of Distillation Without Rationale}

We now state the main limitations that follow from the above view.

\paragraph{(1) Weak transfer of algorithmic structure and intermediate state.}
For tasks where correct outputs depend on maintaining intermediate state (e.g., multi-step
reasoning, symbolic manipulation, long-horizon planning), the rationale $R$ often encodes
the state transitions. When we omit $R$, the student is trained to learn a direct mapping
$X\mapsto Y$ that may be representable only via brittle shortcuts. Consequently, the student
can match teacher accuracy on typical inputs while failing on length generalization,
adversarial perturbations, or compositional variants.

\paragraph{(2) Limited compositional generalization.}
Compositional generalization frequently requires learning reusable subroutines that can be
recombined. In our framework, such subroutines correspond to structure in $R$.
Without rationale supervision, the student has no direct incentive to factor the mapping
into composable parts; it is therefore more likely to entangle features and rely on
distribution-specific correlations.

\paragraph{(3) Poor uncertainty calibration and error detectability.}
Rationales can provide internal consistency checks (e.g., verifying intermediate steps).
If we train only on final answers, the student may replicate the teacher's outputs without
learning the teacher's internal verification logic. This can yield overconfident failures,
especially in low-probability regions of the input distribution.

\paragraph{(4) Capacity bottlenecks are amplified in SLMs.}
Even if the teacher's behavior is learnable in principle, an SLM may lack the information
capacity to represent it. Let $I_{\mathrm{model}}$ denote the information capacity of the
student parameters and let $I_{\mathrm{task}}$ denote the task information complexity. When
\begin{equation}
I_{\mathrm{model}} \;<\; I_{\mathrm{task}},
\label{eq:capacity-mismatch}
\end{equation}
the student cannot fully represent the task, and rationale-free distillation cannot fix
this mismatch. In such cases, rationales, tool use, retrieval, or external memory can
provide additional channels to reduce the effective task complexity seen by the student.

\paragraph{(5) Spurious alignment to teacher idiosyncrasies.}
Since response-only distillation matches outputs rather than the underlying generative
process, the student may inherit teacher-specific quirks, biases, or systematic errors.
Rationales (when faithful) can partially mitigate this by exposing internal steps that can
be corrected or regularized.

\subsection{When Rationale-Free Distillation Is Still Appropriate}

Despite the above limitations, we note that distillation without rationale remains highly
useful when:
(i) the task is low in information complexity,
(ii) inference constraints dominate (latency/memory),
(iii) rationales are unavailable or unreliable,
or (iv) the student will be paired with external modules (retrieval, tools, verifiers)
that offload missing computation and knowledge.

\paragraph{Summary.}
We conclude that distillation without rationale is best understood as transferring a
compressed teacher behavior through a narrow supervision channel. This approach is
effective for many pattern-based or low-entropy tasks, but it is intrinsically limited
for algorithmic and out-of-domain generalization unless augmented with additional
information channels (e.g., rationales, retrieval, tool use, or verification signals).

Nonetheless, in many enterprise domains the SLMs cannot be finetuned to imitate the reasoning of the LLM due to parameter constraints and also due to lack of such rationales from closed-source models. In the next section we provide a proposal to potentially overcome some of these limitations of distilling SLMs without rationales.

\subsection{Enterprise Summary and Practical Implication}

We conclude with an observation directly relevant to enterprise deployments. In many enterprise settings,
a large fraction of operational workloads are \emph{repeatable and routine}: document (image, pdf etc.) parsing and validation. policy compliance checks,
runbook execution, incident triage workflows, knowledge-base question answering, structured ticket routing,
and domain-specific diagnostics. These tasks often admit explicit symbolic structure, for example, as
graphs, rule sets, decision trees, constraint systems, or database queries, for which a sound symbolic
reasoner can be formulated.

Importantly, we can often use a powerful LLM in a \emph{static/offline} manner to help construct this symbolic
layer: the LLM can draft candidate rules, propose schemas, generate synthetic canonical queries and code or assist in
curating the knowledge base and extraction ontology. Once these artifacts are validated and stabilized, we can
deploy an SLM primarily as a lightweight extractor front-end coupled with the symbolic reasoner and external
knowledge. Theorems~\ref{thm:strict-adv-mi} and~\ref{thm:strict-adv-comp} then imply that such hybrid systems
can be strictly more reliable than rationale-free distilled SLMs, particularly for knowledge-intensive and
algorithmic enterprise tasks.



\newpage


\end{document}